\documentclass[11pt]{article}

\usepackage{fullpage,times}

\usepackage{natbib}

\usepackage{amsthm,amsfonts,amsmath,amssymb,epsfig,color,float,graphicx,verbatim}

\newtheorem{theorem}{Theorem}
\newtheorem{proposition}{Proposition}
\newtheorem{lemma}{Lemma}

\newtheorem{definition}{Definition}
\newtheorem{assumption}{Assumption}
\newtheorem{remark}{Remark}
\newtheorem{claim}[theorem]{Claim}

\newcommand{\reals}{\mathbb{R}}
\newcommand{\E}{\mathbb{E}}

\def \Sph{\mathbb{S}^{d-1}}
\def \RR {\mathbb R}

\def \EE {\mathbb E}

\newcommand{\be}{\mathbf{e}}
\newcommand{\bx}{\mathbf{x}}
\newcommand{\bw}{\mathbf{w}}

\newcommand{\bv}{\mathbf{v}}
\newcommand{\bz}{\mathbf{z}}

\newcommand{\Ocal}{\mathcal{O}}

\newcommand{\norm}[1]{\|#1\|}  
\newcommand{\lnorm}[1]{\left\|#1\right\|_{L_2}} 
\newcommand{\lmnorm}[1]{\left\|#1\right\|_{L_2(\mu)}} 
\newcommand{\inner}[1]{\langle#1\rangle}

\newcommand{\secref}[1]{Sec.~\ref{#1}}
\newcommand{\subsecref}[1]{Subsection~\ref{#1}}

\renewcommand{\eqref}[1]{Eq.~(\ref{#1})}
\newcommand{\lemref}[1]{Lemma~\ref{#1}}
\newcommand{\thmref}[1]{Thm.~\ref{#1}}

\newcommand{\ind}[1]{\mathbf{1}\left\{#1\right\}}

\title{The Power of Depth for Feedforward Neural Networks}
\author{Ronen Eldan\\Weizmann Institute of Science\\\texttt{ronen.eldan@weizmann.ac.il}
	\and
	Ohad Shamir\\Weizmann Institute of Science\\\texttt{ohad.shamir@weizmann.ac.il}}

\date{}

\begin{document}

\maketitle

\begin{abstract}
	We show that there is a simple (approximately radial) function on $\reals^d$, expressible by a small 3-layer feedforward neural networks, which cannot be approximated by any 2-layer network, to more than a certain constant accuracy, unless its width is exponential in the dimension. The result holds for virtually all known activation functions, including rectified linear units, sigmoids and thresholds, and formally demonstrates that depth -- even if increased by 1 -- can be exponentially more valuable than width for standard feedforward neural networks. Moreover, compared to related results in the context of Boolean functions, our result requires fewer assumptions, and the proof techniques and construction are very different. 
\end{abstract}

\section{Introduction and Main Result}

Learning via multi-layered artificial neural networks, a.k.a. deep learning, has seen a dramatic resurgence of popularity over the past few years, leading to impressive performance gains on difficult  learning problems, in fields such as computer vision and speech recognition. Despite their practical success, our theoretical understanding of their properties is still partial at best. 

In this paper, we consider the question of the \emph{expressive power} of neural networks of \emph{bounded} size. The boundedness assumption is important here: It is well-known that sufficiently large depth-$2$ neural networks, using reasonable activation functions, can approximate any continuous function on a bounded domain (\cite{cybenko1989approximation,hornik1989multilayer,funahashi1989approximate,barron1994approximation}). However, the required size of such networks can be exponential in the dimension, which renders them impractical as well as highly prone to overfitting. From a learning perspective, both theoretically and in practice, our main interest is in neural networks whose size is bounded. 

For a network of bounded size, a basic architectural question is how to trade off between its width and depth: Should we use networks that are narrow and deep (many layers, with a small number of neurons per layer), or shallow and wide?  Is the ``deep'' in ``deep learning'' really important? Or perhaps we can always content ourselves with shallow (e.g. depth-$2$) neural networks? 


Overwhelming empirical evidence as well as intuition indicates that having depth in the neural network is indeed important: Such networks tend to result in complex predictors which seem hard to capture using shallow architectures, and often lead to better practical performance. However, for the types of networks used in practice, there are surprisingly few formal results  (see related work below for more details).

In this work, we consider fully connected feedforward neural networks, using a linear output neuron and some non-linear activation function on the other neurons, such as the commonly-used rectified linear unit (ReLU, $\sigma(z)=\max\{z,0\}$), as well as the sigmoid ($\sigma(z) = (1+\exp(-z))^{-1}$) and the threshold ($\sigma(z)=\ind{z\geq 0}$). Informally speaking, we consider the following question: \emph{What functions on $\reals^d$ expressible by a network with $\ell$-layers and $w$ neurons per layer, that cannot be well-approximated by any network with $<\ell$ layers, even if the number of neurons is allowed to be much larger than $w$?}

More specifically, we consider the simplest possible case, namely the difficulty of approximating functions computable by $3$-layer networks using $2$-layer networks, when the networks are feedforward and fully connected. Following a standard convention, we define a $2$-layer network of width $w$ on inputs in $\reals^d$ as
\begin{equation}\label{eq:2layer}
\bx\mapsto \sum_{i=1}^{w}v_i\sigma\left(\inner{\bw_i,\bx}+b_i\right)
\end{equation}
where $\sigma:\reals\rightarrow\reals$ is the activation function, and $v_i,b_i\in \reals$, $\bw_i\in\reals^d$, $i=1,\ldots,w$ are parameters of the network. This corresponds to a set of $w$ neurons computing $\bx\mapsto \sigma(\inner{\bw_i,\bx}+b_i)$ in the first layer, whose output is fed to a linear output neuron $\bx\mapsto \sum_{i=1}^{w}v_i x_i$ in the second layer\footnote{Note that sometimes one also adds a constant bias parameter $b$ to the output neuron, but this can be easily simulated by a ``constant'' neuron $i$ in the first layer where $w_i=\mathbf{0}$ and $v_i,b_i$ are chosen appropriately. Also, sometimes the output neuron is defined to have a non-linearity as well, but we stick to linear output neurons, which is a very common and reasonable assumption for networks computing real-valued predictions.}. Similarly, a $3$-layer network of width $w$ is defined as
\begin{equation}\label{eq:3layer}
\sum_{i=1}^{w}u_i\sigma\left(\sum_{j=1}^{w}v_{i,j}\sigma\left(\inner{\bw_{i,j},\bx}+b_{i,j}\right)+c_{i}\right),
\end{equation}
where $u_i,c_i,v_{i,j},b_{i,j}\in\reals,\bw_{i,j}\in \reals^d$, $i,j=1,\ldots,w$ are parameters of the network. Namely, the outputs of the neurons in the first layer are fed to neurons in the second layer, and their outputs in turn are fed to a linear output neuron in the third layer.

Clearly, to prove something on the separation between $2$-layer and $3$-layer networks, we need to make some assumption on the activation function $\sigma(\cdot)$ (for example, if $\sigma(\cdot)$ is the identity, then both $2$-layer and $3$-layer networks compute linear functions, hence there is no difference in their expressive power). All we will essentially require is that $\sigma(\cdot)$ is \emph{universal}, in the sense that a sufficiently large $2$-layer network can approximate any univariate Lipschitz function which is non-constant on a bounded domain. More formally, we use the following assumption:

\begin{assumption}\label{assumption}
	Given the activation function $\sigma$, there is a constant $c_\sigma\geq 1$ (depending only on $\sigma$) such that the following holds: For any $L$-Lipschitz function $f:\reals\rightarrow\reals$ which is constant outside a bounded interval $[-R,R]$, and for any $\delta$, there exist scalars $a,\{\alpha_i,\beta_i,\gamma_i\}_{i=1}^{w}$, where $w\leq c_\sigma\frac{RL}{\delta}$, such that the function
	\[
	h(x) = a+\sum_{i=1}^{w}\alpha_i\cdot\sigma(\beta_i x-\gamma_i)
	\]
	satisfies
	\[
	\sup_{x\in \reals} \left|f(x)-h(x)\right| \leq \delta.
	\]
\end{assumption}

This assumption is satisfied by the standard activation functions we are familiar with. First of all, we provide in Appendix \ref{sec:relu} a constructive proof for the ReLU function. For the threshold, sigmoid, and more general sigmoidal functions (e.g. monotonic functions which satisfy $\lim_{z\rightarrow\infty}\sigma(z)=a,\lim_{z\rightarrow-\infty}\sigma(z)=b$ for some $a\neq b$ in $\reals$), the proof idea is similar, and implied by the proof of Theorem 1 of  \cite{debao1993degree}\footnote{Essentially, a single neuron with such a sigmoidal activation can express a (possibly approximate) single-step function, a combination of $w$ such neurons can express a function with $w$ such steps, and any $L$-Lipschitz function which is constant outside $[-R,R]$ can be approximated to accuracy $\delta$ with a function involving $\Ocal(RL/\delta)$ steps.}.
Finally, one can weaken the assumed bound on $w$ to any $\text{poly}(R,L,1/\delta)$, at the cost of a worse polynomial dependence on the dimension $d$ in \thmref{thm:main} part $1$ below (see \subsecref{subsec:approximability} for details).

In addition, for technical reasons, we will require the following mild growth and measurability conditions, which are satisfied by virtually all activation functions in the literature, including the examples discussed earlier:
\begin{assumption}\label{assumption2}
	The activation function $\sigma$ is (Lebesgue) measurable and satisfies
	$$
	|\sigma(x)| \leq C(1 + |x|^{\alpha}) 
	$$
	for all $x \in \RR$ and for some constants $C,\alpha > 0$.
\end{assumption}

Our main result is the following theorem, which implies that there are $3$-layer networks of width polynomial in the dimension $d$, which cannot be arbitrarily well approximated by $2$-layer networks, unless their width is exponential in $d$:

\begin{theorem}\label{thm:main}
	Suppose the activation function $\sigma(\cdot)$ satisfies assumption \ref{assumption} with constant $c_{\sigma}$, as well as assumption \ref{assumption2}.
	Then there exist universal constants $c,C>0$ such that the following holds: For every dimension $d>C$, there is a probability measure $\mu$ on $\RR^d$ and a function $g:\reals^d\rightarrow \reals$ with the following properties:
	\begin{enumerate}
	\item $g$ is bounded in $[-2,+2]$, supported on $\{\bx:\norm{\bx}\leq C\sqrt{d}\}$, and expressible by a $3$-layer network of width $Cc_{\sigma}d^{19/4}$.
	\item Every function $f$, expressed by a $2$-layer network of width at most $c e^{c d}$, satisfies
	\[
	\E_{\bx\sim\mu}\left(f(\bx)-g(\bx)\right)^2 \geq c.
	\]
	\end{enumerate}
\end{theorem}

The proof is sketched in \secref{sec:idea}, and is formally presented in \secref{sec:proof}. Roughly speaking, $g$ approximates a certain radial function $\tilde{g}$, depending only on the norm of the input. With $3$ layers, approximating radial functions (including $\tilde{g}$) to arbitrary accuracy is straightforward, by first approximating the squared norm function, and then approximating the univariate function acting on the norm. However, performing this approximation with only $2$ layers is much more difficult, and the proof shows that exponentially many neurons are required to approximate $\tilde{g}$ to more than constant accuracy. We conjecture (but do not prove) that a much wider family of radial functions also satisfy this property. 

We make the following additional remarks about the theorem:

\begin{remark}[Activation function]
The theorem places no constraints on the activation function $\sigma(\cdot)$ beyond assumptions \ref{assumption} and \ref{assumption2}. In fact, the inapproximability result for the function $\tilde{g}$ holds even if the activation functions are different across the first layer neurons, and even if they are chosen adaptively (possibly depending on $\tilde{g}$), as long as they satisfy assumption \ref{assumption2}.
\end{remark}

\begin{remark}[Constraints on the parameters]
The theorem places no constraints whatsoever on the  parameters of the $2$-layer networks, and they can take any values in $\reals$. This is in contrast to related depth separation results in the context of threshold circuits, which do require the size of the parameters to be constrained (see discussion of related work below).
\end{remark}

\begin{remark}[Properties of $g$]
At least for specific activation functions such as the ReLU, sigmoid, and threshold, the proof construction implies that $g$ is $\text{poly}(d)$-Lipschitz, and the $3$-layer network expressing it has parameters bounded by $\text{poly}(d)$. 
\end{remark}


\subsection*{Related Work}

On a qualitative level, the question we are considering is similar to the question of Boolean circuit lower bounds in computational complexity: In both cases, we consider functions  which can be represented as a combination of simple computational units (Boolean gates in computational complexity; neurons in neural networks), and ask how large or how deep this representation needs to be, in order to compute or approximate some given function. For Boolean circuits, there is a relatively rich literature and some  strong lower bounds. A recent example is the paper \cite{rossman2015average}, which shows for any $d\geq 2$ an explicit depth $d$, linear-sized circuit on $\{0,1\}^n$, which cannot be non-trivially approximated by depth $d-1$ circuits of size polynomial in $n$. That being said, it is well-known that the type of computation performed by each unit in the circuit can crucially affect the hardness results, and lower bounds for Boolean circuits do \emph{not} readily translate to neural networks of the type used in practice, which are real-valued and express continuous functions. For example, a classical result on Boolean circuits states that the parity function over $\{0,1\}^d$ cannot be computed by constant-depth Boolean circuits whose size is polynomial in $d$ (see for instance \cite{hastad1986almost}). Nevertheless, the parity function can in fact be easily computed by a simple $2$-layer, $\Ocal(d)$-width \emph{real-valued} neural network with most reasonable activation functions\footnote{See \cite{rumelhart1986parallel}, Figure 6, where reportedly the structure was even found automatically by back-propagation. For a threshold activation function $\sigma(z)=\ind{z\geq 0}$ and input $\bx=(x_1,\ldots,x_d)\in \{0,1\}^d$, the network is given by $
	\bx~\mapsto~\sum_{i=1}^{d+1}(-1)^{i+1}\sigma\left(\sum_{j=1}^{d}x_j-i+\frac{1}{2}\right)
	$.
In fact, we only need $\sigma$ to satisfy $\sigma(z)=1$ for $z\geq \frac{1}{2}$ and $\sigma(z)=0$ for $z\leq -\frac{1}{2}$, so the construction easily generalizes to other activation functions (such as a ReLU or a sigmoid), possibly by using a small linear combination of them to represent such a $\sigma$.}. 

A model closer to ours is a \emph{threshold circuit}, which is a neural network where all neurons (including the output neuron) has a threshold activation function, and the input is from the Boolean cube (see \cite{parberry1994circuit} for a survey). For threshold circuits, the main known result in our context is that computing inner products mod $2$ over $d$-dimensional Boolean vectors cannot be done with a $2$-layer network with $\text{poly}(d)$-sized parameters and $\text{poly}(d)$ width, but can be done with a small $3$-layer network (\cite{hajnal1993threshold}). Note that unlike neural networks in practice, the result in \cite{hajnal1993threshold} is specific to the non-continuous threshold activation function, and considers hardness of exact representation of a function by $2$-layer circuits, rather than merely approximating it. Following the initial publication of our paper, we were informed (\cite{martens2015}) that the proof technique, together with techniques in the papers (\cite{maass1994comparison,martens2013representational})), can possibly be used to show that inner product mod $2$ is also hard to approximate, using $2$-layer neural networks with continuous activation functions, as long as the network parameters are constrained to be polynomial in $d$, and that the activation function satisfies certain regularity conditions\footnote{See  remark $20$ in \cite{martens2013representational}. These conditions are needed for constructions relying on distributions over a finite set (such as the Boolean hypercube). However, since we consider continuous distributions on $\reals^d$, we do not require such conditions.}. Even so, our result does not pose any constraints on the parameters, nor regularity conditions beyond assumptions \ref{assumption},\ref{assumption2}. Moreover, we introduce a new proof technique which is very different, and demonstrate hardness of approximating not the Boolean inner-product-mod-$2$ function, but rather functions in $\reals^d$ with a simple geometric structure (namely, radial functions).

Moving to networks with real-valued outputs, one related field is arithmetic circuit complexity (see \cite{shpilka2010arithmetic} for a survey), but the focus there is on computing polynomials, which can be thought of as neural networks where each neuron computes a linear combination or a product of its inputs. Again, this is different than most standard neural networks used in practice, and the results and techniques do not readily translate.

Recently, several works in the machine learning community attempted to address questions similar to the one we consider here. \cite{pascanu2013number,montufar2014number} consider the number of linear regions which can be expressed by ReLU networks of a given width and size, and \cite{bianchinicomplexity} consider the topological complexity (via Betti numbers) of networks with certain activation functions, as a function of the depth. Although these can be seen as measures of the function's complexity, such results do not translate directly to a lower bound on the approximation error, as in \thmref{thm:main}.  \cite{delalleau2011shallow,martens2014expressive} and \cite{cohen2015expressive} show strong approximation hardness results for  certain neural network architectures (such as polynomials or representing a certain tensor structure), which are however fundamentally different than the standard neural networks considered here. 

Quite recently, \cite{telgarsky2015representation} gave a simple and elegant construction showing that for any $k$, there are $k$-layer, $\Ocal(1)$ wide ReLU networks on one-dimensional data, which can express a sawtooth function on $[0,1]$ which oscillates $\Ocal(2^k)$ times, and moreover, such a rapidly oscillating function cannot be approximated by $\text{poly}(k)$-wide ReLU networks with $o(k/\log(k))$ depth. This also implies regimes with exponential separation, e.g. that there are $k^2$-depth networks, which any approximating $k$-depth network requires $\Omega(\exp(k))$ width. These results demonstrate the value of depth for arbitrarily deep, standard ReLU networks, for a single dimension and using functions which have an exponentially large Lipschitz parameter. In this work, we use different techniques, to show exponential separation results for general activation functions, even if the number of layers changes by just $1$ (from two to three layers), and using functions in $\reals^d$ whose Lipschitz parameter is polynomial in $d$. 

\section{Proof Sketch}\label{sec:idea}

In a nutshell, the $3$-layer network we construct approximates a radial function with bounded support (i.e. one which depends on the input $\bx$ only via its Euclidean norm $\norm{\bx}$, and is $0$ for any $\bx$ whose norm is larger than some threshold). With $3$ layers, approximating radial functions is rather straightforward: First, using assumption \ref{assumption}, we can construct a linear combination of neurons expressing the univariate mapping $z\mapsto z^2$ arbitrarily well in any bounded domain. Therefore, by adding these combinations together, one for each coordinate, we can have our network first compute (approximately) the mapping $\bx\mapsto \norm{\bx}^2=\sum_i x_i^2$ inside any bounded domain, and then use the next layer to compute some univariate function of $\norm{\bx}^2$, resulting in an approximately radial function. With only $2$ layers, it is less clear how to approximate such radial functions. Indeed, our proof essentially indicates that approximating radial functions with $2$ layers can require exponentially large width. 

To formalize this, note that if our probability measure $\mu$ has a well-behaved density function which can be written as $\varphi^2(\bx)$ for some function $\varphi$, then the approximation guarantee in the theorem, $\E_{\mu}(f(\bx)-g(\bx))^2$, can be equivalently written as 
\begin{equation}\label{eq:fgg}
\int (f(\bx)-g(\bx))^2 \varphi^2(\bx)d\bx ~=~ \int(f(\bx)\varphi(\bx)-g(\bx)\varphi(\bx))^2d\bx ~=~ \norm{f\varphi-g\varphi}_{L_2}^2.
\end{equation}
In particular, we will consider a density function which equals $\varphi^2(\bx)$, where $\varphi$ is the inverse Fourier transform of the indicator  $\ind{\bx\in B}$, $B$ being the origin-centered unit-volume Euclidean ball (the reason for this choice will become evident later). Before continuing, we note that a formula for $\varphi$ can be given explicitly (see \lemref{lem:varphi}), and an illustration of it in $d=2$ dimensions is provided in Figure \ref{fig:varphi}. Also, it is easily verified that $\varphi^2(\bx)$ is indeed a density function: It is clearly non-negative, and by isometry of the Fourier transform, $\int \varphi^2(\bx)d\bx=\int \hat{\varphi}^2(\bx)d\bx=\int \ind{\bx\in B}^2 d\bx$, which equals $1$ since $B$ is a unit-volume ball.

\begin{figure}\label{fig:varphi}
	\centering
	\includegraphics[scale=0.45]{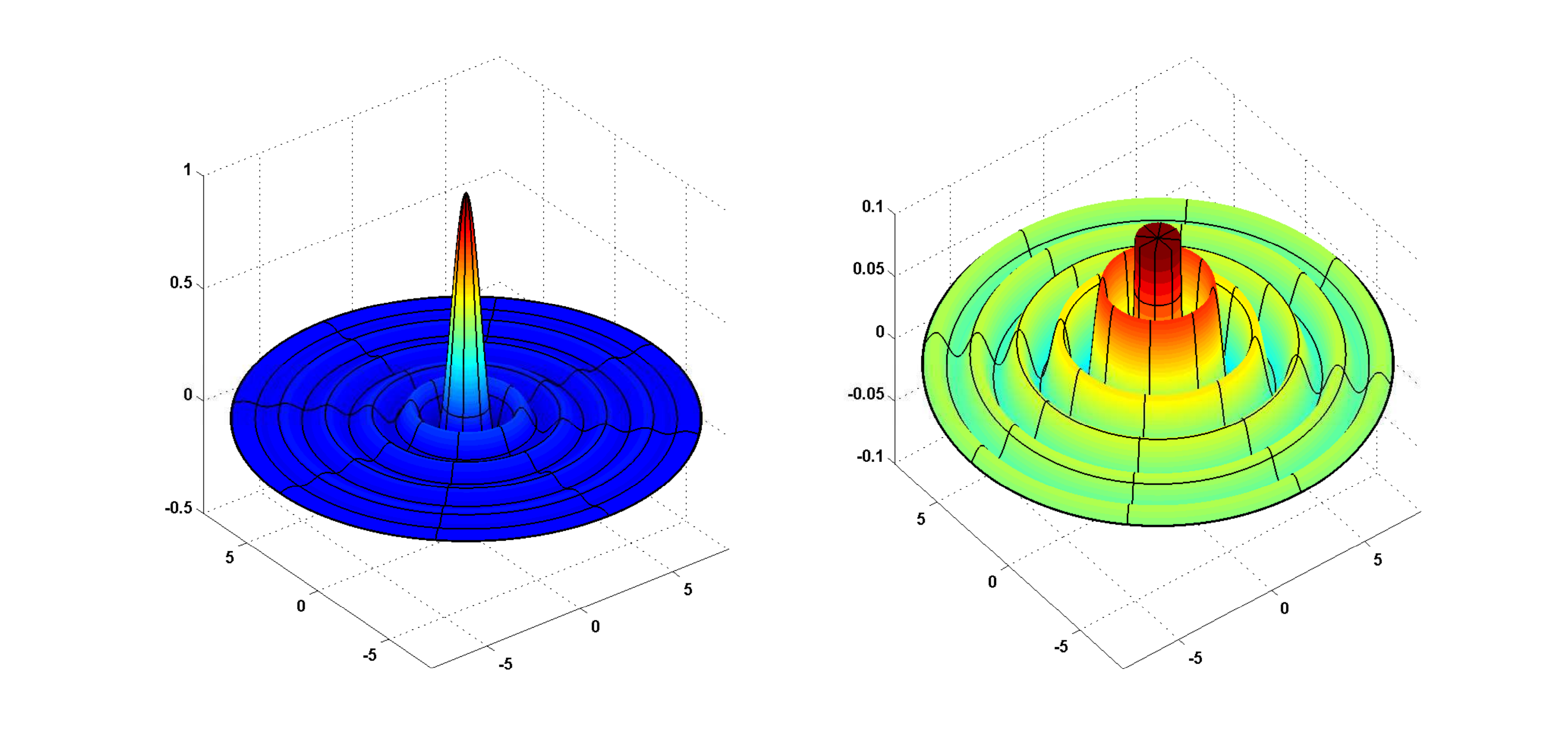}
	\caption{The left figure represents $\varphi(\bx)$ in $d=2$ dimensions. The right figure represents a cropped and re-scaled version, to better show the oscillations of $\varphi$ beyond the big origin-centered bump. The density of the probability measure $\mu$ is defined as $\varphi^2(\cdot)$}
\end{figure}

Our goal now is to lower bound the right hand side of \eqref{eq:fgg}. To continue, we find it convenient to consider the Fourier transforms $\widehat{f\varphi},\widehat{g\varphi}$ of the functions $f\varphi,g\varphi$, rather than the functions themselves. Since the Fourier transform is isometric, the above equals
\[
\norm{\widehat{f\varphi}-\widehat{g\varphi}}_{L_2}^2.
\]
Luckily, the Fourier transform of functions expressible by a $2$-layer network has a very particular form. Specifically, consider any function of the form
$$
f(\bx) = \sum_{i=1}^k f_i(\inner{\bv_i,\bx}),
$$
where $f_i: \RR \to \RR$ (such as $2$-layer networks as defined earlier). Note that $f$ may not be square-integrable, so formally speaking it does not have a Fourier transform in the standard sense of a function on $\reals^d$. However, assuming $|f_i(x)|$ grows at most polynomially as $x\rightarrow \infty$ or $x\rightarrow -\infty$, it does have a Fourier transform in the more general sense of a tempered distribution (we refer the reader to the proof for a more formal discussion). This distribution can be shown to be supported on $\bigcup_i \mathrm{span}\{\bv_i\}$: namely, a finite collection of lines\footnote{Roughly speaking, this is because each function $\bx\mapsto f_i(\inner{\bv_i,\bx})$ is constant in any direction perpendicular to $\bv_i$, hence do not have non-zero Fourier components in those directions. In one dimension, this can be seen by the fact that the Fourier transform of the constant $0$ function is the Dirac delta function, which equals $0$ everywhere except at the origin.}. The convolution-multiplication principle implies that $\widehat{f\varphi}$ equals $\hat{f}\star\hat{\varphi}$, or the convolution of $\hat{f}$ with the indicator of a unit-volume ball $B$. Since $\hat{f}$ is supported on $\bigcup_i \mathrm{span}\{\bv_i\}$, it follows that
$$
\mathrm{Supp}( \widehat{f \varphi} ) ~~\subseteq~~ T~~ := \bigcup_{i=1}^k \left ( \mathrm{span}\{\bv_i\} + B \right ).
$$
In words, the support of $\widehat{f\varphi}$ is contained in a union of tubes of bounded radius passing through the origin.
This is the key property of $2$-layer networks we will use to derive our main theorem. Note that it holds regardless of the exact shape of the $f_i$ functions, and hence our proof will also hold if the activations in the network are different across the first layer neurons, or even if they are chosen in some adaptive manner.

To establish our theorem, we will find a function $g$ expressible by a $3$-layer network, such that $\widehat{g\varphi}$ has a constant distance (in $L_2$ space) from any function supported on $T$ (a union of $k$ tubes as above).
Here is where high dimensionality plays a crucial role: Unless $k$ is exponentially large in the dimension, the domain $T$ is very sparse when one considers large distances from the origin, in the sense that
$$
\frac{Vol_{d-1}(T \cap r \Sph)}{Vol_{d-1}(r \Sph)} \lesssim k e^{-d}
$$
(where $\Sph$ is the $d$-dimensional unit Euclidean sphere, and $Vol_{d-1}$ is the $d-1$-dimensional Hausdorff measure) whenever $r$ is large enough with respect to the radius of $B$. Therefore, we need to find a function $g$ so that $\widehat{g\varphi}$ has a lot of mass far away from the origin, which will ensure that $\norm{\widehat{f\varphi}-\widehat{g\varphi}}_{L_2}^2$ will be large. Specifically, we wish to find a function $g$ so that $g\varphi$ is radial (hence $\widehat{g\varphi}$ is also radial, so having large mass in any direction implies large mass in all directions), and has a significant \emph{high-frequency} component, which implies that its Fourier transform has a significant  portion of its mass outside of the ball $rB$.

The construction and analysis of this function constitutes the technical bulk of the proof. The main difficulty in this step is that even if the Fourier transform $\hat{g}$ of $g$ has some of its $L_2$ mass on high frequencies, it is not clear that this will also be true for $\widehat{g \varphi} = g \star \ind{B}$ (note that while convolving with a Euclidean ball increases the average distance from the origin in the $L_1$ sense, it doesn't necessarily do the same in the $L_2$ sense). 

We overcome this difficulty by considering a random superposition of indicators of thin shells: Specifically, we consider the function
\begin{equation}\label{eq:intuitive}
\tilde{g}(\bx) = \sum_{i=1}^{N}\epsilon_i g_i(\bx),
\end{equation}
where $\epsilon_i \in \{-1,+1\}$, $N=\text{poly}(d)$, and $g_i(\bx)=\ind{\norm{\bx}\in \Delta_i}$, where $\Delta_i$ are disjoint intervals of width $\Ocal(1/N)$ on values in the range $\Theta(\sqrt{d})$. Note that strictly speaking, we cannot take our hard-to-approximate function $g$ to equal $\tilde{g}$, since $\tilde{g}$ is discontinuous and therefore cannot be expressed by a $3$-layer neural network with continuous activations functions. However, since our probability distribution $\varphi^2$ can be shown to have bounded density in the support of \eqref{eq:intuitive}, we can use a $3$-layer network to approximate such a function arbitrarily well with respect to the distribution $\varphi^2$ (for example, by computing $\sum_{i}\epsilon_i g_i(\bx)$ as above, with each hard indicator function $g_i$ replaced by a Lipschitz function, which differs from $g_i$ on a set with arbitrarily small probability mass). Letting the function $g$ be such a good approximation, we get that if no $2$-layer network can approximate the function $\tilde{g}$ in \eqref{eq:intuitive}, then it also cannot approximate its $3$-layer approximation $g$.

Let us now explain why the function defined in \eqref{eq:intuitive} gives us what we need. For large $N$, each $g_i$ is supported on a thin Euclidean shell, hence $g_i\varphi$ is approximately the same as $c_i g_i$ for some constant $c_i$. As a result, $\tilde{g}(\bx) \varphi(\bx) \approx \sum_{i=1}^{N}\epsilon_i c_i g_i(\bx)$, so its Fourier transform (by linearity) is $\widehat{\tilde{g} \varphi}(\bw)\approx \sum_{i=1}^{N}\epsilon_i c_i \hat{g_i}(\bw)$. Since $g_i$ is a simple indicator function, its Fourier transform $\hat{g_i}(\bw)$ is not too difficult to compute explicitly, and involves an appropriate Bessel function which turns out to have a sufficiently large mass sufficiently far away from the origin.

Knowing that each summand $g_i$ has a relatively large mass on high frequencies, our only remaining objective is to find a choice for the signs $\epsilon_i$ so that the entire sum will have the same property. This is attained by a random choice of signs: it is an easy observation that given an orthogonal projection $P$ in a Hilbert space $H$, and any sequence of vectors $v_1,...,v_N \in H$ such that $|P v_i| \geq \delta |v_i|$, one has that $\EE \left [ |P \sum_i \epsilon_i v_i|^2 \right ] \geq \delta^2 \sum_i |v_i|^2$ when the signs $\epsilon_i$ are independent Bernoulli $\pm 1$ variables. Using this observation with $P$ being the projection onto the subspace spanned by functions supported on high frequencies and with the functions $\hat g_i$, it follows that there is at least one choice of the $\epsilon_i$'s so that a sufficiently large portion of $\tilde{g}$'s mass is on high frequencies.

\section{Preliminaries}\label{sec:prelim}

We begin by defining some of the standard notation we shall use. We let $\mathbb{N}$ and $\reals$ denote the natural and real numbers, respectively. Bold-faced letters denote vectors in $d$-dimensional Euclidean space $\reals^d$, and plain-faced letters to denote either scalars or functions (distinguishing between them should be clear from context).  $L_2$ denotes the space of squared integrable functions ($\int_{\bx}f^2(\bx)d\bx<\infty$, where the integration is over $\reals^d$), and $L_1$ denotes the space of absolutely integrable functions ($\int_{\bx}|f(\bx)|d\bx<\infty$). $\norm{\cdot}$ denotes the Euclidean norm, $\inner{\cdot,\cdot}_{L_2}$ denotes inner product in $L_2$ space (for functions $f,g$, we have $\inner{f,g}_{L_2}=\int f(\bx)g(\bx)d\bx$), $\lnorm{\cdot}$ denotes the standard norm in $L_2$ space ( $\lnorm{f}^2=\int_{\bx}f(\bx)^2d\bx$), and $\lmnorm{\cdot}$ denotes the $L_2$ space norm weighted by a probability measure $\mu$ (namely $\lmnorm{f}^2=\int f(\bx)^2d\mu(\bx)$). Given two functions $f,g$, we let $fg$ be shorthand for the function $\bx\mapsto f(\bx)\cdot g(\bx)$, and $f+g$ be shorthand for $\bx\mapsto f(\bx)+g(\bx)$. Given two sets $A,B$ in $\reals^d$, we let $A+B=\{a+b:a\in A,b\in B\}$ and $A^C=\{a\in \reals^d:a\notin A\}$

\textbf{Fourier Transform.} For a function $f:\RR \to \RR$, our convention for the Fourier transform is
$$
\hat{f}(w) = \int_{\reals} \exp\left(-2\pi i x w\right) f(x) dx
$$
whenever the integral is well defined. This is generalized for $f: \RR^d \to \RR$ by
\begin{equation}
\hat{f}(\bw) = \int_{\reals^d} \exp\left(-2\pi i \inner{\bx,\bw}\right)f(\bx)d\bx.\label{eq:fourier}
\end{equation}

\textbf{Radial Functions.} A radial function $f:\reals^d\mapsto\reals$ is such that $f(\bx)=f(\bx')$ for any $\bx,\bx'$ such that $\norm{\bx}=\norm{\bx'}$. When dealing with radial functions, which are invariant to rotations, we will somewhat abuse notation and interchangeably use vector arguments $\bx$ to denote the value of the function at $\bx$, and scalar arguments $r$ to denote the value of the same function for any vector of norm $r$. Thus, for a radial function $f:\reals^d\rightarrow \reals$, $f(r)$ equals $f(\bx)$ for any $\bx$ such that $\norm{\bx}=r$.

\textbf{Euclidean Spheres and Balls.} Let $\Sph$ be the unit Euclidean sphere in $\reals^d$, $B_d$ be the $d$-dimensional unit Euclidean ball, and let $R_d$ be the radius so that $R_d B_d$ has volume one. By standard results, we have the following useful lemma:

\begin{lemma}\label{lem:Rd}
	$R_d = \sqrt{\frac{1}{\pi}}\left(\Gamma\left(\frac{d}{2}+1\right)\right)^{1/d}$, which is always between $\frac{1}{5}\sqrt{d}$ and $\frac{1}{2}\sqrt{d}$.
\end{lemma}

\textbf{Bessel Functions.} Let $J_{\nu}:\reals\mapsto \reals$ denote the Bessel function of the first kind, of order $\nu$. The Bessel function has a few equivalent definitions, for example $J_{\nu}(x)=\sum_{m=0}^{\infty}\frac{(-1)^m}{m!\Gamma(m+\nu+1)}\left(\frac{x}{2}\right)^{2m+\nu}$ where $\Gamma(\cdot)$ is the Gamma function. Although it does not have a closed form, $J_{\nu}(x)$ has an oscillating shape, which for asymptotically large $x$ behaves as $\sqrt{\frac{2}{\pi x}}\cos\left(-\frac{(2\nu+1)\pi}{4}+x\right)$. Figure \ref{fig:bessel} illustrates the function for $\nu=20$. In appendix \ref{sec:bessel}, we provide additional results and approximations for the Bessel function, which are necessary for our proofs. 

\begin{figure}\label{fig:bessel}
	\centering
	\includegraphics[scale=0.5]{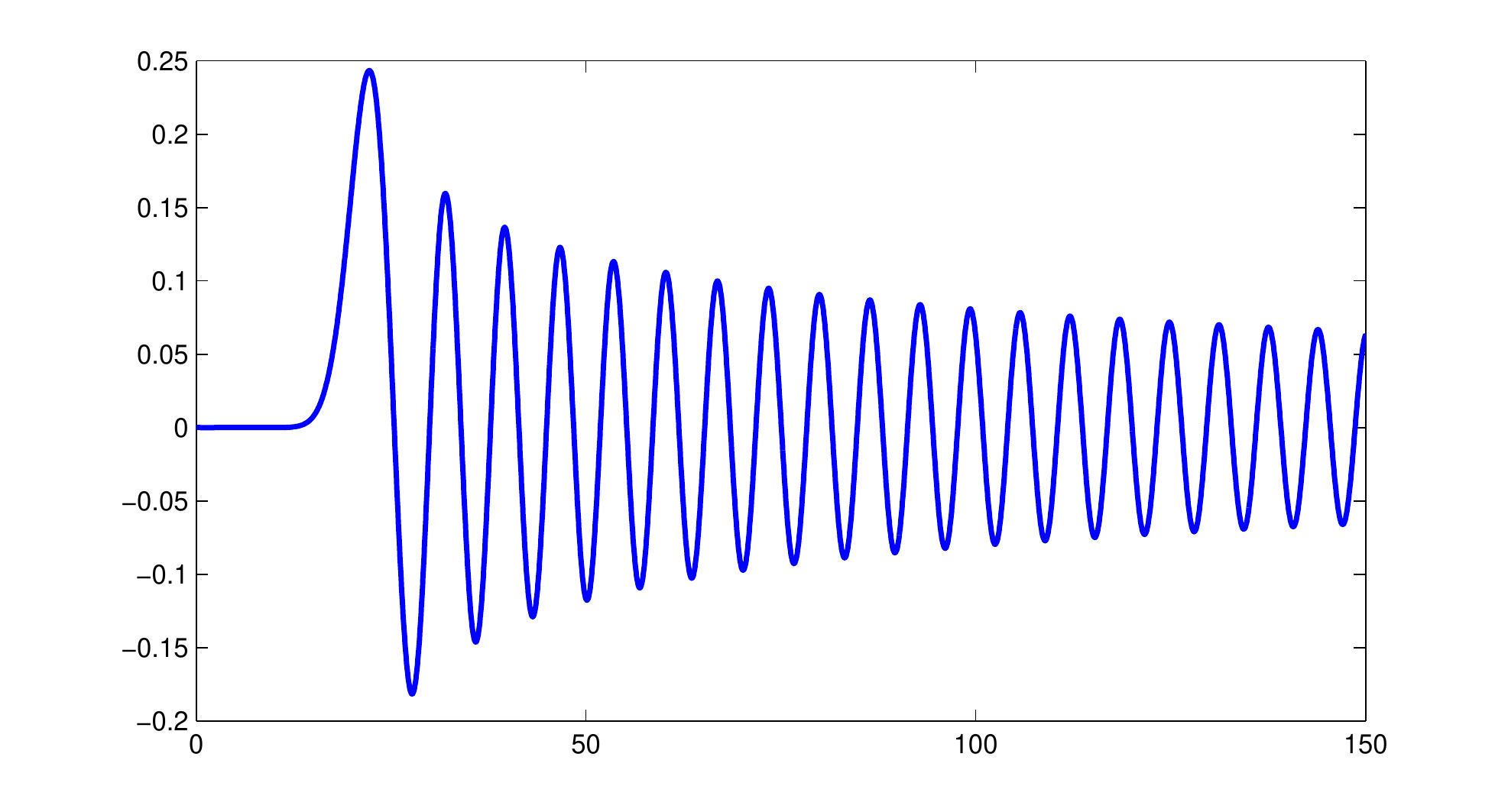}
	\caption{Bessel function of the first kind, $J_{20}(\cdot)$}
\end{figure}

\section{Proof of \thmref{thm:main}}\label{sec:proof}

In this section, we provide the proof of \thmref{thm:main}. Note that some technical proofs, as well as some important technical lemmas on the structure of Bessel functions, are deferred to the appendix.

\subsection{Constructions}

As discussed in \secref{sec:idea}, our theorem rests on constructing a distribution $\mu$ and an appropriate function $g$, which is easy to approximate (w.r.t. $\mu$) by small $3$-layer networks, but difficult to approximate using $2$-layer networks. Thus, we begin by formally defining $g,\mu$ that we will use.

First, $\mu$ will be defined as the measure whose density is $\frac{d \mu}{d\bx} = \varphi^2(\bx)$, where $\varphi(\bx)$ is the Fourier transform of the indicator of a unit-volume Euclidean ball $\ind{\bw\in R_d B_d}$. Note that since the Fourier transform is an isometry, $\int_{\reals^d}\varphi(\bx)^2 d\bx = \int_{\reals^d}\ind{\bw\in R_d B_d}^2 d\bw=1$, hence $\mu$ is indeed a probability measure. The form of $\varphi$ is expressed by the following lemma:
\begin{lemma}\label{lem:varphi}
	Let $\varphi(\bx)$ be the Fourier transform of $\ind{\bw\in R_d B_d}$. Then 
	\[
	\varphi(\bx) = \left(\frac{R_d}{\norm{\bx}}\right)^{d/2}J_{d/2}(2\pi R_d\norm{\bx}).
	\]
\end{lemma}
The proof appears in Appendix \ref{subsec:prooflemvarphi}.

To define our hard-to-approximate function, we introduce some notation. Let $\alpha\geq 1$ and $\gamma$ be some large numerical constants to be determined later, and set $N=\gamma d^{2}$, assumed to be an integer (essentially, we need $\alpha,\gamma$ to be sufficiently large so that all the lemmas we construct below would hold). Consider the intervals
\[
\Delta_i = \left[\left(1+\frac{i-1}{N}\right)\alpha\sqrt{d}~,~\left(1+\frac{i}{N}\right)\alpha\sqrt{d}\right]~~,~~ i=1,2,\ldots,N.
\]
We split the intervals to ``good'' and ``bad'' intervals using the following definition:
\begin{definition}
	$\Delta_i$ is a \emph{good interval} (or equivalently, $i$ is good) if for any $x\in \Delta_i$
	\[
	J_{d/2}^2(2\pi R_d x)\geq \frac{1}{80\pi R_d x}.
	\]
	Otherwise, we say that $\Delta_i$ is a \emph{bad interval}.
\end{definition}

For any $i$, define 
\begin{equation} \label{eq:defgi}
g_i(x) = \begin{cases}\ind{x\in \Delta_i} & \text{$i$ good}\\ 0 & \text{$i$ bad}\end{cases}
\end{equation}
By definition of a ``good'' interval and \lemref{lem:varphi}, we see that $g_i$ is defined to be non-zero, when the value of $\varphi$ on the corresponding interval $\Delta_i$ is sufficiently bounded away from $0$, a fact which will be convenient for us later on.

Our proof will revolve around the $L_2$ function 
\[
\tilde{g}(\bx)=\sum_{i=1}^{N} \epsilon_i g_i(\bx),
\]
which as explained in \secref{sec:idea}, will be shown to be easy to approximate arbitrarily well with a $3$-layer network, but hard to approximate with a $2$-layer network.


\subsection{Key Lemmas}

In this subsection, we collect several key technical lemmas on $g_i$ and $\varphi$, which are crucial for the main proof. The proofs of all the lemmas can be found in Appendix \ref{sec:technicalproofs}.

The following lemma ensures that $\varphi(\bx)$ is sufficiently close to being a constant on any good interval:
\begin{lemma}\label{lem:flat}
	If $d\geq 2$, $\alpha\geq c$ and $N\geq c\alpha^{3/2}d^2$ (for some sufficiently large universal constant $c$), then inside any good interval $\Delta_i$, $\varphi(x)$ has the same sign, and
	\[
	\frac{\sup_{x\in \Delta_i}|\varphi(x)|}{\inf_{x\in \Delta_i}|\varphi(x)|} \leq 1+d^{-1/2}.
	\]
\end{lemma}

The following lemma ensures that the Fourier transform $\hat{g}_i$ of $g_i$ has a sufficiently large part of its $L_2$ mass far away from the origin:
\begin{lemma} \label{lem:nothinsh}
	Suppose $N\geq 100\alpha d^{3/2}$. Then for any $i$,
	\[
	\int_{(2R_dB_d)^{C}}\hat{g_i}^2(\bw) d\bw ~\geq~ \frac{1}{2} \int_{\reals^d}\hat{g_i}^2(\bw)d\bw,
	\]
	where $\hat{g_i}$ is the Fourier transform of $g_i$.
\end{lemma}

The following lemma ensures that $\widehat{g_i\varphi}$ also has sufficiently large $L_2$ mass far away from the origin:
\begin{lemma} \label{lem:nothinsh2}
	Suppose that $\alpha\geq C$, $N\geq C \alpha^{3/2}d^2$ and $d>C$, where $C>0$ is a universal constant. Then for any $i$,
	$$	
	\int_{(2 R_d B_d)^C} (\widehat{(g_i \varphi)} (\bw))^2 d \bw \geq \frac{1}{4} \int_{\reals^d} (\varphi (\bx) g_i (\bx))^2 d\bx.
	$$	
\end{lemma}

The following lemma ensures that a linear combination of the $g_i$'s has at least a constant $L_2(\varphi^2)$ probability mass.
\begin{lemma} \label{lem:bigmass}
	Suppose that $\alpha\geq c$ and $N\geq c(\alpha d)^{3/2}$ for some sufficiently large universal constant $c$, then for every choice of $\epsilon_i \in \{-1,+1\}$, $i=1,\ldots,N$, one has
	\[
	\int\left(\sum_{i=1}^{N}\epsilon_i g_i(\bx)\right)^2\varphi^2(\bx)d\bx\geq \frac{0.003}{\alpha}.
	\]
\end{lemma}

Finally, the following lemma guarantees that the non-Lipschitz function $\sum_{i=1}^{N}\epsilon_i g_i(\bx)$ can be approximated by a Lipschitz function (w.r.t. the density $\varphi^2$). This will be used to show that $\sum_{i=1}^{N}\epsilon_i g_i(\bx)$ can indeed be approximated by a $3$-layer  network.
\begin{lemma}\label{lem:lipapprox}
	Suppose that $d\geq 2$. For any choice of $\epsilon_i \in \{-1,+1\}$, $i=1,\ldots,N$, there exists an $N$-Lipschitz function $f$, supported on $[\alpha \sqrt{d}, 2 \alpha \sqrt{d}]$ and with range in $[-1,+1]$, which satisfies
	\[
	\int\left(f(\bx)-\sum_{i=1}^{N}\epsilon_i g_i(\bx)\right)^2\varphi^2(\bx)d\bx~\leq~\frac{3}{\alpha^2\sqrt{d}}.
	\]
\end{lemma}

\subsection{Inapproximability of the Function $\tilde{g}$ with $2$-Layer Networks}

The goal of this section is to prove the following proposition.
\begin{proposition} \label{prop:main}
	Fix a dimension $d$, suppose that $d>C$, $\alpha > C$ and $N \geq C \alpha^{3/2} d^2$ and let $k$ be an integer satisfying
	\begin{equation} \label{eq:ksmall}
	k \leq c e^{c d}
	\end{equation}
	with $c,C>0$ being universal constants.
	There exists a choice of $\epsilon_i \in \{-1,+1\}$, $i=1,\ldots,N$, such that the function
	$\tilde{g}(\bx) = \sum_{i=1}^{N}\epsilon_i g_i(\norm{\bx})$
	has the following property. Let $f:\RR^d \to \RR$ be of the form
	\begin{equation} \label{eq:formf}
	f(\bx) = \sum_{i=1}^k f_i(\langle \bx, \bv_i \rangle)
	\end{equation}
	for $\bv_i \in \Sph$, where $f_i: \RR \to \RR$ are measurable functions satisfying
	$$
	|f_i(x)| \leq C'(1 + |x|^{\kappa})
	$$
	for constants $C',\kappa > 0$. Then one has
	$$
	\lmnorm{f - \tilde{g}} \geq \delta / \alpha
	$$	
	where $\delta > 0$ is a universal constant.
\end{proposition}

The proof of this proposition requires a few intermediate steps. In the remainder of the section, we will assume that $N,d,\alpha$ are chosen to be large enough to satisfy the assumptions of Lemma \ref{lem:bigmass} and Lemma \ref{lem:nothinsh2}. In other words we assume that $d>C$ and $N \geq C \alpha^{3/2} d^2$ for a suitable universal constant $C>0$. We begin with the following:

\begin{lemma} \label{lem:signschoice}
	Suppose that $d,N$ are as above. There exists a choice of $\epsilon_i \in \{-1,1\}$, $i=1,\ldots,N$ such that 
	$$
	\int_{(2 R_d B_d)^C} \left (\widehat{ \left (\sum_i \epsilon_i g_i \varphi \right )  } (\bw) \right )^2 d \bw  \geq c
	$$
	for a universal constant $c>0$.
\end{lemma}
\begin{proof}
	Suppose that each $\epsilon_i$ is chosen independently and uniformly at random from $\{-1,+1\}$. It suffices to show that	
	$$
	\EE \left [\int_{(2 R_d B_d)^C} \left (\widehat{ \left (\sum_i \epsilon_i g_i \varphi \right )  } (\bw) \right )^2 d \bw  \right ] \geq c
	$$
	for some universal constant $c>0$, since that would ensure there exist some choice of $\epsilon_1,\ldots,\epsilon_N$ satisfying the lemma statement. Define $h(\bw) = \ind{\bw \in (2 R_d B_d)^C}$ and consider the operator 
	$$
	P(g) = \widehat{\hat g h}.
	$$
	This is equivalent to removing low-frequency components from $g$ (in the Fourier domain), and therefore is an orthogonal projection. 
	According to Lemma \ref{lem:nothinsh2} and isometry of the Fourier transform, we have 
	\begin{equation} \label{eq:lem44}
	\lnorm{P(g_i \varphi)}^2\geq \frac{1}{4}\lmnorm{g_i}^2
	\end{equation}
	for every good $i$. Moreover, an application of Lemma \ref{lem:bigmass}, and the fact that $\inner{g_i,g_j}_{L_2}=0$ for any $i\neq j$ (as $g_i,g_j$ have disjoint supports) tells us that
	\begin{equation} \label{eq:lem45}
	\sum_{i=1}^{N} \lmnorm{g_i}^2 = \lmnorm{\sum_{i=1}^{N}g_i}^2 \geq c
	\end{equation}
	for a universal constant $c>0$. We finally get,
	\begin{align*}
	\EE \left [\int_{(2 R_d B_d)^C} \left (\widehat{ \left (\sum_{i=1}^{N} \epsilon_i g_i \varphi \right ) } (\bw) \right )^2 d \bw  \right ] ~&= \EE \left [\int_{\reals^d}\left (  \left (\sum_{i=1}^{N} \epsilon_i P(g_i \varphi)  \right ) (\bx) \right )^2 d \bx  \right ] \\
	& = \EE \lnorm{\sum_{i=1}^{N} \epsilon_i P (g_i \varphi)}^2 \\
	& = \sum_{i,j=1}^{N} \EE[\epsilon_i \epsilon_j] \langle P (g_i \varphi), P (g_j \varphi) \rangle_{L_2} \\
	&= \sum_{i=1}^N\lnorm{P (g_i \varphi)}^2\\
	& \stackrel{\eqref{eq:lem44}}{\geq} \frac{1}{4} \sum_{i,j=1}^{N} \lmnorm{ g_i}^2\\
	& \stackrel{\eqref{eq:lem45}}{\geq} c/4.
	\end{align*}
	for a universal constant $c>0$.
\end{proof}

\begin{claim} \label{claim:tubes}
	Let $f$ be a function such that $f\varphi \in L_2$, and is of the form in \eqref{eq:formf}. Suppose that the functions $f_i$ are measurable functions satisfying
	\begin{equation} \label{tempered}
	|f_i(x)| \leq C(1 + |x|^{\alpha})
	\end{equation}
	for constants $C,\alpha > 0$. Then,
	\begin{equation} \label{tubesupport}
	\mathrm{Supp} (\widehat{f \varphi}) ~\subset~ \bigcup_{i=1}^k \left ( \mathrm{Span}\{\bv_i\} + R_d B_d \right )
	\end{equation}
\end{claim}
\begin{proof}
	Informally, the proof is based on the convolution-multiplication and linearity principles of the Fourier transform, which imply that if  $f=\sum_i f_i$, where each $f_i$ as well as $\varphi$ have a Fourier transform, then $\widehat{f\varphi}=\sum_i \widehat{f_i\varphi}=\sum_i \hat{f_i}\star \hat{\varphi}$. Roughly speaking, in our case each $\hat{f_i}(\bx)=\hat{f_i}(\inner{\bx,\bv_i})$ (as a function in $\reals^d$) is shown to be supported on $\mathrm{Span}\{\bv_j\}$, so its convolution with $\hat{\varphi}$ (which is an indicator for the ball $R_dB_d$) must be supported on $\mathrm{Span}\{\bv_i\} + R_d B_d$. Summing over $i$ gives the stated result. 
	
	Unfortunately, this simple analysis is not formally true, since we are not guaranteed that $f_i$ has a Fourier transform as a function in $L_2$ (this corresponds to situations where the integral in the definition of the Fourier transform in \eqref{eq:fourier} does not converge). However, at least for functions satisfying the claim's conditions, the Fourier transform still exists as a generalized function, or \emph{tempered distribution}, over $\reals^d$, and using this object we can attain the desired result.
	
	We now turn to provide the formal proof and constructions, starting with a description of tempered distributions and their relevant properties (see \cite[Chapter 11]{HNBook} for a more complete survey). To start, let $\mathcal{S}$ denote the space of Schwartz functions \footnote{This corresponds to infinitely-differentiable functions $\psi:\reals^d\mapsto\reals$ with super-polynomially decaying values and derivatives, in the sense that $\sup_{\bx}\left( x_1^{\alpha_1}\cdot x_2^{\alpha_2}\cdots x_d^{\alpha_d}\frac{\partial^{\alpha_1}}{\partial x_1^{\alpha_1}}\cdot\frac{\partial^{\alpha_2}}{\partial x_2^{\alpha_2}}\cdots\frac{\partial^{\alpha_d}}{\partial x_d^{\alpha_d}}f(\bx)\right)<\infty$ for all indices $\alpha_1,\ldots,\alpha_d$.} on $\reals^d$. A tempered distribution $\mu$ in our context is a continuous linear operator from $\mathcal{S}$ to $\reals$ (this can also be viewed as an element in the dual space $\mathcal{S}^*$). In particular, any measurable function $h:\reals^d\mapsto\reals$, which satisfies a polynomial growth condition similar to \eqref{tempered}, can be viewed as a tempered distribution defined as
	\[
	\psi\mapsto \inner{h,\psi} ~:=~ \int_{\reals^d} h(\bx)\psi(\bx)d\bx,
	\]
	where $\psi\in \mathcal{S}$. Note that the growth condition ensures that the integral above is well-defined. The Fourier transform $\hat{h}$ of a tempered distribution $h$ is also a tempered distribution, and defined as 
	\[
	\inner{\hat{h},\psi}~:=~\inner{h,\hat{\psi}},
	\]
	where $\hat{\psi}$ is the Fourier transform of $\psi$. It can be shown that this directly generalizes the standard notion of Fourier transforms of functions. Finally, we say that a tempered distribution $h$ is supported on some subset of $\reals^d$, if $\inner{h,\psi}=0$ for any function $\psi\in\mathcal{S}$ which vanishes on that subset.
	
	With these preliminaries out of the way, we turn to the setting considered in the claim. Let $\hat{f}_i$ denote the Fourier transform of $f_i$ (possibly as a tempered distribution, as described above), the existence of which is guaranteed by the fact that $f_i$ is measurable and by \eqref{tempered}. 
	We also define, for $\psi: \RR^d \to \RR$ and $1 \leq i \leq N$, a corresponding function $\psi_i: \RR \to \RR$ by 
	\[
	\psi_i(x) = \psi(x \bv_i),
	\]
	and define the tempered distribution $\mu_i$ (over Schwartz functions in $\reals^d$) as
	$$
	\inner{\mu_i,\psi}~:=~\inner{f_i,\hat{\psi}_i},
	$$
	which is indeed an element of $\mathcal{S}^*$ by the linearity of the Fourier transform, by the continuity of $\psi \to \psi_i(x)$ with respect to the topology of $\mathcal{S}$ and by the dominated convergence theorem. Finally, define
	\[
	\tilde f_i(\bx) = f_i(\langle \bx, \bv_i \rangle).
	\]
	Using the fact that\footnote{This is because $\int \hat{g}(\bx)d\bx = \int\int g(\bx)\exp(-2\pi i \inner{\bx,\bw})d\bw d\bx = \int g(\bx)\left(\int \exp(-2\pi i \inner{\bx,\bw})\cdot 1 d\bw\right)d\bx = \int g(\bx)\delta(\bx)d\bx=g(0)$, where $\delta(\cdot)$ is the Dirac delta function, which is the Fourier transform of the constant $1$ function. See also \cite[Chapter 11, Example 11.31]{HNBook}.}
	\begin{equation} \label{eq:fourier0}
	\int_{\RR^d} \hat g(\bx) d\bx = g(0) 
	\end{equation}
	for any $g\in \mathcal{S}$, recalling that $\bv_i$ has unit norm, and letting $\bv_i^\perp$ denote the subspace of vectors orthogonal to $\bv_i$, we have the following for any $\psi\in\mathcal{S}$:
	\begin{align} \label{eq:cm2}
	\inner{\tilde{f}_i,\hat{\psi}}&=\int_{\reals_d}\tilde{f}_i(\bx)\hat{\psi}(\bx)d\bx ~=~ \int_\RR \int_{\bv_i^\perp} \tilde f_i(\bx + \bv_i y) \hat \psi(\bx + \bv_i y) d\bx dy  \\
	& = \int_\RR f_i(y) \int_{\bv_i^\perp} \hat \psi(\bx + \bv_i y) d\bx dy \nonumber \\
	& = \int_\RR f_i(y) \int_{\bv_i^\perp} \int_{\RR^d} \psi(\bw)  \exp(-2 \pi i \langle \bw, \bx + \bv_i y \rangle ) d\bw d\bx dy \nonumber \\
	& = \int_\RR f_i(y) \int_{\bv_i^\perp} \int_{\RR} \int_{\bv_i^\perp} \psi(\bw_1 + \bv_i w_2)  \exp(-2 \pi i \langle \bw_1 + \bv_i w_2, \bx + \bv_i y \rangle ) d\bw_1 d w_2 d\bx dy \nonumber\\
	& = \int_\RR f_i(y) \int_{\RR} \exp(-2 \pi i w_2 y ) \int_{\bv_i^\perp}  \int_{\bv_i^\perp} \psi(\bw_1 + \bv_i w_2)  \exp(-2 \pi i \langle \bw_1 , \bx \rangle ) d\bw_1  d\bx d w_2 dy \nonumber\\
	& \stackrel{(\ref{eq:fourier0})}{=} \int_\RR f_i(y) \int_{\RR} \exp(-2 \pi i w_2 y ) \psi(\bv_i w_2)  d w_2 dy \nonumber\\
	& = \int_\RR f_i(y) \hat \psi(\bv_i y)  dy ~=~ \int_\RR f_i(y) \hat \psi_i(y)  dy ~=~ \inner{f_i,\hat{\psi}_i}~=~\inner{\mu_i,\psi},\nonumber
	\end{align}
	where the use of Fubini's theorem is justified by the fact that $\psi \in \mathcal{S}$.
	
	We now use the convolution-multiplication theorem (see e.g., \cite[Theorem 11.35]{HNBook}) according to which if $f,g \in L_1$ then
	\begin{equation} \label{eq:cm}
	\widehat{f \star g} = \hat f \hat g.
	\end{equation}

	Using this, we have the following for every $\psi \in \mathcal{S}$:
	\begin{align*}
	\inner{\widehat{\tilde f_i \varphi},\psi}   ~&= \inner{\tilde{f}_i\varphi,\hat{\psi}}~=~ \inner{\tilde{f}_i,\varphi\hat{\psi}}\\
	& \stackrel{(\ref{eq:cm})}{=} \inner{\tilde{f}_i,\widehat { \hat \varphi \star \psi }}\\
	& \stackrel{(\ref{eq:cm2})}{=} \inner{\mu_i,\hat{\varphi}\star\psi}~=~ \inner{\mu_i,\ind{R_d B_d} \star \psi}.
	\end{align*}
	Based on this equation, we now claim that $\inner{\widehat{\tilde{f}_i\varphi},\psi}=0$ for any $\psi\in\mathcal{S}$ supported on the complement of $\mathrm{Span}\{\bv_i\} + R_d B_d$. This would imply that the tempered distribution $\widehat{\tilde{f}_i\varphi}$ is supported in $\mathrm{Span}\{\bv_i\} + R_d B_d$, and therefore $\widehat{f\varphi}$ is supported in $\bigcup_{i=1}^k \left ( \mathrm{Span}\{\bv_i\} + R_d B_d \right )$ (by linearity of the Fourier transform and the fact that $f=\sum_{i=1}^{k}\tilde{f}_i$). Since the Fourier transform of $f\varphi$ as a tempered distribution coincides with the standard one (as we assume $f\varphi\in L_2$), the result follows. 
	
	It remains to prove that $\widehat{f_i\varphi}(\psi)=0$ for any $\psi\in\mathcal{S}$ supported on the complement of $\mathrm{Span}\{\bv_i\} + R_d B_d$. For such $\psi$, by definition of a convolution, $\ind{R_d B_d} \star \psi$ is supported on the complement of $\mathrm{Span}\{\bv_i\}$. However, $\mu_i$ is supported on $\mathrm{Span}\{\bv_i\}$ (since  if $\psi$ vanishes on $\bv_i$, then $\psi_i$ is the zero function, hence $\hat{\psi}_i$ is also the zero function, and $\inner{\mu_i,\psi}=\inner{f_i,\hat{\psi}_i}=0$). Therefore, $\inner{\mu_i,\ind{R_d B_d} \star \psi}=0$, which by the last displayed equation, implies that $\inner{\widehat{\tilde f_i \varphi},\psi}=0$ as required.
\end{proof}

\begin{lemma} \label{lem:conctubes}
	Let $q,w$ be two functions of unit norm in $L_2$. Suppose that $q$ satisfies 
	\begin{equation} \label{tubesupport2}
	\mathrm{Supp} (q) ~\subset~ \bigcup_{j=1}^k \left ( \mathrm{Span}\{\bv_j\} + R_d B_d \right )
	\end{equation}
	for some $k \in \mathbb{N}$. Moreover, suppose that $w$ is radial and that $\int_{2 R_d B_d} w(\bx)^2 d\bx \leq 1 - \delta$ for some $\delta\in [0,1]$. Then
	$$
	\langle q, w \rangle_{L_2} \leq 1-\delta/2 + k \exp(-cd)
	$$  
	where $c>0$ is a universal constant.
\end{lemma}
\begin{proof}
	Define $A = (2 R_d B_d)^C$ and denote 
	$$
	T = \bigcup_{j=1}^k \left ( \mathrm{Span}\{\bv_j\} + R_d B_d \right )
	$$
	so that $T$ contains the support of $q(\cdot)$. For each $r>0$, define 
	$$
	h(r) = \frac{Vol_{d-1} (r \Sph \cap T ) }{Vol_{d-1} (r \Sph)},
	$$
	where $\Sph$ is the Euclidean sphere in $\reals^d$. Since $T$ is star shaped, the function $h(r)$ is non-increasing. We claim that there exists a universal constant $c>0$ such that
	\begin{equation} \label{eq:tube}
	h(2 R_d) \leq k \exp(-cd).
	\end{equation}
	Indeed, fix $\bv \in \Sph$ and consider the tube $T_0 = \mathrm{Span}\{\bv\} + R_d B_d$. Let $\bz$ be a uniform random point in $2 R_d \Sph$. We have by a standard calculation (See e.g., \cite[Section 2]{Sodin05})
	\begin{align*}
	\Pr( \bz \in T_0 ) ~& = \Pr( \norm{\bz}^2 - \inner{\bz, \bv}^2 \leq R_d^2) \\
	&= \Pr(4R_d^2-\inner{\bz,\bv}^2\leq R_d^2)
	~=~ \Pr( |\langle \bz, \bv \rangle | \geq \sqrt{3} R_d)  \\
	& =  \frac{\int_{\sqrt{3}/2}^{1} (1-t^2)^{(d-3)/2} } { \int_{0}^{1} (1-t^2)^{(d-3)/2}} dt \leq \exp(-cd).
	\end{align*}
	Using a union bound and the definition of $h$, equation \eqref{eq:tube} follows. 
	
	Next, define 
	$$
	\tilde q(\bx) = \frac{\int_{\norm{\bx}\Sph} q(y) d \mathcal{H}_{d-1}(y) }{Vol_{d-1} (\norm{\bx} \Sph)}
	$$
	to be the averaging of $q(\cdot)$ with respect to rotations (in the above formula $\mathcal{H}_{d-1}$ denotes the $d-1$ dimensional Hausdorff measure, i.e. the standard measure in $d-1$ dimensions). We have the following: Since $w(\cdot)$ is radial and has unit $L_2$ norm, and we assume $q(\cdot)$ is supported on $T$, we have
	\begin{align*}
	\int_A w(\bx) q(\bx) d\bx ~& \stackrel{(1)}{=} \int_A w(\bx) \tilde q(\bx) d\bx \\
	& \stackrel{(2)}{\leq} \lnorm{w} \lnorm{\tilde q \ind{A}} \\
	& \stackrel{(3)}{=} \sqrt{\int_{2 R_d}^\infty \tilde q(r)^2 Vol_{d-1} (r \Sph) dr} \\
	& = \sqrt{\int_{2 R_d}^\infty Vol_{d-1} (r \Sph) \left ( \frac{1}{Vol_{d-1} (r \Sph)} \int_{r \Sph \cap T} q(y) d \mathcal{H}_{d-1}(y) \right )^2 dr} \\
	& = \sqrt{\int_{2 R_d}^\infty h(r)^2 Vol_{d-1} (r \Sph) \left (\frac{1}{Vol_{d-1}(r \Sph \cap T)} \int_{r \Sph \cap T} q(y) d \mathcal{H}_{d-1}(y) \right )^2 dr} \\
	& \stackrel{(4)}{\leq} \sqrt{\int_{2 R_d}^\infty h(r)^2 Vol_{d-1} (r \Sph) \left (\frac{1}{Vol_{d-1}(r \Sph \cap T)} \int_{r \Sph \cap T} q(y)^2 d \mathcal{H}_{d-1}(y) \right ) dr} \\
	& = \sqrt{ \int_{2 R_d}^\infty h(r) \int_{r \Sph} q^2(y) d \mathcal{H}_{d-1}(y) dr} \\
	& \stackrel{(5)}{\leq} \sqrt{h(2 R_d)} \lnorm{q \ind{A}} \stackrel{(6)}{\leq} k \exp(-cd /2).
	\end{align*}
	In the above, (1) follows from $w(\cdot)$ being radial; (2) follows from Cauchy-Schwartz; (3) follows from $w(\cdot)$ having unit $L_2$ norm; (4) follows from the fact that the term being squared is the expectation of $q(y)$ where $y$ is uniformly distributed in $r\Sph\cap T$, and we have 
	$(\E_{y}[q(y)])^2\leq \E_{y}[q^2(y)]$ by Jensen's inequality; (5) follows from $h(\cdot)$ being non-increasing; and (6) follows from \eqref{eq:tube} and the fact that $q(\cdot)$ has unit $L_2$ norm.
	
	As a result of these calculations, we have
	\begin{align*}
	\langle q, w \rangle_{L_2} ~& = \int_A w(\bx) q(\bx) d\bx + \int_{A^C} w(\bx) q(\bx) d\bx \\
	& \leq k \exp(-cd /2) + \lnorm{q} \lnorm{w \ind{A^C}} =  k\exp(-cd/2) + 1\cdot \sqrt{\int_{A^C}w^2(\bx)d\bx}\\
	&\leq k\exp(-cd/2)+\sqrt{1 - \delta}.
	\end{align*}
	where we used the assumption that $q(\cdot)$ is unit norm and that $\int_{A^C}w^2(\bx)d\bx=\int_{(2R_dB_d)^C}w^2(\bx)d\bx\leq 1-\delta$. Since $\sqrt{1-\delta}\leq 1-\frac{1}{2}\delta$ for any $\delta\in [0,1]$, the result follows.
\end{proof}

\begin{proof}[Proof of Proposition \ref{prop:main}]
	Define 
	$$
	\tilde{g}(\bx) = \sum_i \epsilon_i g_i(|\bx|)
	$$
	where $(\epsilon_i)$ are the signs provided by Lemma \ref{lem:signschoice}. According to Lemma \ref{lem:bigmass}, we have 
	\begin{equation} \label{eq:hphibig}
	\lmnorm{\tilde{g}} \geq  c_1 / \alpha
	\end{equation}
	for a universal constant $c_1>0$. Note that the function $\tilde{g}$ is bounded between $-1$ and $1$. Define the function $w = \frac{\widehat{\tilde{g} \varphi}}{\lnorm{\tilde{g} \varphi}}$. By construction (hence according to Lemma \ref{lem:signschoice}) we have
	\begin{align*}
	\int_{2 R_d B_d} w(\bx)^2 d\bx ~& = 1 - \frac{\int_{2 R_d B_d} \widehat{\tilde{g} \varphi}(\bx)^2 d\bx}{ \lnorm{\tilde{g} \varphi}^2 }  \\
	& \leq 1 - \frac{\int_{2 R_d B_d} \widehat{\tilde{g} \varphi}(\bx)^2 d\bx}{ \lnorm{\varphi}^2 } \leq 1 - c_2,
	\end{align*}
	for a universal constant $c_2 > 0$, where in the first inequality we used the fact that $|\tilde{g}(\bx)|\leq 1$ for all $\bx$. 
	
	Next, define the function $q = \frac{\widehat{f \varphi}}{\lnorm{f \varphi}}$, where $f$ is an arbitrary function having the form in \eqref{eq:formf}. Thanks to the assumptions on the functions $f_i$, we may invoke\footnote{Claim \ref{claim:tubes} also requires that $f\varphi$ is an $L_2$ function, but we can assume this without loss of generality: Otherwise, $\lnorm{f\varphi}=\infty$, and since $\tilde{g}\varphi$ is an $L_2$ function with $\lnorm{\tilde{g}\varphi}<\infty$, we would have $\lmnorm{f - \tilde{g}}^2=\lnorm{f\varphi-\tilde{g}\varphi}^2=\infty$, in which case the proposition we wish to prove is trivially true.} Claim \ref{claim:tubes}, by which we observe that the functions $w,q$ satisfy the assumptions of Lemma \ref{lem:conctubes}. Thus, as a consequence of this lemma we obtain that
	\begin{equation} \label{eq:prodqw}
	\inner{q, w}_{L_2} \leq 1 - c_2/2 + k \exp(-c_3d)
	\end{equation}
	for a universal constant $c_3>0$. Next, we claim that since $\lnorm{q} = \lnorm{w} = 1$, we have for every scalars $\beta_1,\beta_2>0$ that
	\begin{equation} \label{eq:stupidtriangle}
	\lnorm{\beta_1 q - \beta_2 w} \geq \frac{\beta_2}{2} \lnorm{q-w}.
	\end{equation}
	Indeed, we may clearly multiply both $\beta_1$ and $\beta_2$ by the same constant affecting the correctness of the formula, thus we may assume that $\beta_2 = 1$. It thus amounts to show that for two unit vectors $v,u$ in a Hilbert space, one has that $\min_{\beta > 0} \norm{\beta v - u}^2 \geq \frac{1}{4} \norm{v-u}^2$.
	We have 
	\begin{align*}
	\min_\beta \norm{\beta v - u}^2 ~& = \min_\beta \left (\beta^2 \norm{v}^2 - 2 \beta \langle v,u \rangle + \norm{u}^2 \right ) \\
	& = \min_\beta \left (\beta^2 - 2 \beta \langle v,u \rangle + 1 \right ) \\
	& = 1 - \langle v, u \rangle^2 = \frac{1}{2} \norm{v-u}^2
	\end{align*}
	which in particular implies formula \eqref{eq:stupidtriangle}.
	
	Combining the above, and using the fact that $q,w$ have unit $L_2$ norm, we finally get
	\begin{align*}
	\lmnorm{f - \tilde{g}} ~& = \lnorm{f\varphi-\tilde{g}\varphi} =
	\lnorm{\widehat{f\varphi}-\widehat{\tilde{g}\varphi}} = \lnorm{ \left(\lnorm{f \varphi}\right)q(\cdot) -  \left(\lnorm{\tilde{g} \varphi}\right)w(\cdot)} \\
	& \stackrel{\eqref{eq:stupidtriangle}}{\geq}  \frac{1}{2} \lnorm{ q - w } \lnorm{\tilde{g} \varphi} ~=~ \frac{1}{2}\lnorm{q-w}\lmnorm{\tilde{g}}\\
	& \stackrel{ \eqref{eq:hphibig} }{\geq}  \frac{1}{2}\sqrt{2 (1 - \inner{q,w}_{L_2})}\frac{c_1}{ \alpha} \\
	& \stackrel{\eqref{eq:prodqw}}{\geq} \frac{c_1}{2 \alpha} \sqrt{ 2 \max(c_2/2 - k \exp(-c_3d), 0)} \geq \frac{c_1 \sqrt{c_2}}{4 \alpha}
	\end{align*}
	where in the last inequality, we use the assumption in \eqref{eq:ksmall}, choosing $c=\min\{c_2/4,c_3\}$. The proof is complete.
\end{proof}

\subsection{Approximability of the Function $\tilde{g}$ with $3$-Layer Networks}\label{subsec:approximability}
	
The next ingredient missing for our proof is the construction of a $3$-layer function which approximates the function $\tilde{g}=\sum_{i=1}^{N}\epsilon_i g_i$.
\begin{proposition} \label{prop:approx}
There is a universal constant $C>0$ such that the following holds. Let $\delta\in (0,1)$. Suppose that $d\geq C$ and that the functions $g_i$ are constructed as in \eqref{eq:defgi}. For any choice of $\epsilon_i \in \{-1,+1\}$, $i=1,\ldots,N$, there exists a function $g$ expressible by a $3$-layer network of width at most $\frac{8c_{\sigma}}{\delta}\alpha^{3/2}N d^{11/4}+1$, and with range in $[-2,+2]$, such that
\[
\lmnorm{ g(\bx)-\sum_{i=1}^{N}\epsilon_i g_i(\norm{\bx}) } ~\leq~\frac{\sqrt{3}}{\alpha d^{1/4}} + \delta.
\]	
\end{proposition}

The proof of this proposition relies on assumption \ref{assumption}, which ensures that we can approximate univariate functions using our activation function. As discussed before \thmref{thm:main}, one can also plug in weaker versions of the assumption (i.e. worse polynomial dependence of the width $w$ on $R,L,1/\delta$), and get versions of proposition \ref{prop:approx} where the width guarantee has worse polynomial dependence on the parameters $N,\alpha,d,\delta$. This would lead to versions of the \thmref{thm:main} with somewhat worse constants and polynomial dependence on the dimension $d$. 

For this proposition, we need a simple intermediate result, in which an approximation for radially symmetric Lipschitz functions in $\RR^d$, using assumption \ref{assumption}, is constructed. 

\begin{lemma} \label{lem:threelayerapprox}
	Suppose the activation function $\sigma$ satisfies assumption \ref{assumption}. Let $f$ be an $L$-Lipschitz function supported on $[r,R]$, where $r\geq 1$. Then for any $\delta>0$, there exists a 
	function $g$ expressible by a $3$-layer network of width at most $\frac{2c_{\sigma}d^2 R^2L}{\sqrt{r}\delta}+1$,
	such that
	$$
	\sup_{\bx\in\RR^d} \bigl |g(\bx) - f(\norm{\bx}) \bigr | < \delta.
	$$
\end{lemma}

\begin{proof}
	Define the $2R$-Lipschitz function 
	\[
	l(x)=\min\{x^2,R^2\},
	\] 
	which is constant outside $[-R,R]$, as well as the function
	\[
	\ell(\bx)=\sum_{i=1}^d l(x_i)=\sum_{i=1}^{d}\min\{x_i^2,R^2\}
	\]
	on $\reals^d$. Applying assumption \ref{assumption} on $l(\cdot)$, we can obtain a
	function $\tilde{l}(x)$ having the form $a+\sum_{i=1}^{w}\alpha_i\sigma(\beta_i x-\gamma_i)$ so that
	$$
	\sup_{x\in\reals}\left | \tilde{l}(x) - l(x) \right | \leq \frac{\sqrt{r}\delta}{dL},
	$$
	and where the width parameter $w$ is at most $ \frac{2c_{\sigma}dR^2L}{\sqrt{r}\delta}$.
	Consequently, the function 
	\[
	\tilde{\ell}(\bx)=\sum_{i=1}^{d}\tilde{l}(x_i)
	\]
	can be expressed in the form $a+\sum_{i=1}^{w}\alpha_i\sigma(\beta_i x-\gamma_i)$ where $w\leq \frac{2c_{\sigma}d^2R^2L}{\sqrt{r}\delta}$, and it holds that
	\begin{equation} \label{eq:reluap1}
	\sup_{\bx\in\reals^d}\left | \tilde{\ell}(\bx) - \ell(\bx) \right | \leq \frac{\sqrt{r}\delta}{L}.
	\end{equation} 
	Next, define 
	\[
	s(x) = \begin{cases}f(\sqrt{x})&x\geq 0\\ 0& x<0\end{cases}.
	\]
	 Since $f$ is $L$-Lipschitz and supported on $\{x:r\leq x\leq R\}$, it follows that $s$ is $\frac{L}{2\sqrt{r}}$-Lipschitz and supported on the interval $[-R^2,R^2]$. Invoking assumption \ref{assumption} again, we can construct a function $\tilde{s}(x)=a+\sum_{i=1}^{w}\alpha_i\sigma(\beta_i x-\gamma_i)$ satisfying 
	\begin{equation} \label{eq:reluap2}
	\sup_{x\in \reals} |\tilde{s}(x) - s(x)| < \delta/2,
	\end{equation}	
	where $w\leq \frac{c_{\sigma}R^2L}{\sqrt{r}\delta}$.
	
	Now, let us consider the composition $g = \tilde{s} \circ \tilde{\ell}$, which by definition of $\tilde{s},\tilde{\ell}$, has the form
	\begin{equation}\label{eq:gform}
	a+\sum_{i=1}^{w}u_i\sigma\left(\sum_{j=1}^{w}v_{i,j}\sigma\left(\inner{\bw_{i,j},\bx}+b_{i,j}\right)+c_{i}\right)
	\end{equation}
	for appropriate scalars $a,u_i,c_i,v_{i,j},b_{i,j}$ and vectors $\bw_{i,j}$, and where $w$ is at most
	\[ \max\left\{\frac{2c_{\sigma}d^2R^2L}{\sqrt{r}\delta},\frac{c_{\sigma}R^2L}{\sqrt{r}\delta}\right\} ~=~ \frac{2c_{\sigma}d^2 R^2L}{\sqrt{r}\delta}.
	\]
	\eqref{eq:gform} is exactly a $3$-layer network (compare to \eqref{eq:3layer}), except that there is an additional constant term $a$. However, by increasing $w$ by $1$, we can simulate $a$ by an additional neuron $\bx\mapsto \frac{a}{\sigma(\sigma(0)+z)}\cdot \sigma(\sigma(\inner{\mathbf{0},\bx})+z)$, where $z$ is some scalar such that $\sigma(\sigma(0)+z)\neq 0$ (note that if there is no such $z$, then $\sigma$ is the zero function, and therefore cannot satisfy assumption \ref{assumption}). So, we can write the function $g$ as a $3$-layer network (as defined in \eqref{eq:3layer}), of width at most 
	\[
	\frac{2c_{\sigma}d^2 R^2L}{\sqrt{r}\delta}+1.
	\]	
	All the remains now is to prove that $\sup_{\bx\in\reals^d}|g(\bx)-f(\norm{\bx})|\leq \delta$. To do so, we note that for any $\bx\in\reals^d$, we have
	\begin{align*}
	|g(\bx)-f(\norm{\bx})|&= \left|\tilde{s}(\tilde{\ell}(\bx))-f(\norm{\bx})\right|\\
	&\leq \left|\tilde{s}(\tilde{\ell}(\bx))-s(\tilde{\ell}(\bx)\right|
	+\left|s(\tilde{\ell}(\bx)-s(\ell(\bx))\right|+\left|s(\ell(\bx))-f(\norm{\bx})\right|\\
	&= \left|\tilde{s}(\tilde{\ell}(\bx))-s(\tilde{\ell}(\bx)\right|
	+\left|s(\tilde{\ell}(\bx)-s(\ell(\bx))\right|+\left|f(\sqrt{\ell(\bx)})-f(\norm{\bx})\right|.	
	\end{align*}
	Let us consider each of the three absolute values:
	\begin{itemize}
		\item The first absolute value term is at most $\delta/2$ by  \eqref{eq:reluap2}.
		\item The second absolute value term, since $s$ is $\frac{L}{2\sqrt{r}}$-Lipschitz, is at most $\frac{L}{2\sqrt{r}}|\tilde{\ell}(\bx)-\ell(\bx)|$, which is at most $\delta/2$ by \eqref{eq:reluap1}. 
		\item As to the third absolute value term, if $\norm{\bx}^2\leq R^2$, then $\ell(\bx)=\norm{\bx}^2$ and the term is zero. If $\norm{\bx}^2>R^2$, then it is easy to verify that $\ell(\bx)\geq R^2$, and since $f$ is continuous and supported on $[r,R]$, it follows that $f(\sqrt{\ell(\bx)}=f(\norm{\bx})=0$, and again, we get a zero. 
	\end{itemize}
	Summing the above, we get that $|g(\bx)-f(\norm{\bx})|\leq \frac{\delta}{2}+\frac{\delta}{2}+0 = \delta$ as required.
\end{proof}

We are now ready to prove Proposition \ref{prop:approx}, which is essentially a combination of Lemmas \ref{lem:lipapprox} and \ref{lem:threelayerapprox}.

\begin{proof}[Proof of Proposition \ref{prop:approx}]
	First, invoke Lemma \ref{lem:lipapprox} to obtain an $N$-Lipschitz function $h$ with range in $[-1,+1]$ which satisfies
	\begin{equation} \label{eq:est-lip1}
	\lmnorm{h(\bx)-\sum_{i=1}^{N}\epsilon_i g_i(\bx)}~=~\sqrt{\int_{\reals^d}\left(\tilde f(\bx)-\sum_{i=1}^{N}\epsilon_i g_i(\bx)\right)^2\varphi^2(\bx)d\bx}~\leq~\frac{\sqrt{3}}{\alpha d^{1/4}}.
	\end{equation}
	Next, we use Lemma \ref{lem:threelayerapprox} with $R = 2 \alpha \sqrt{d}$, $r = \alpha \sqrt{d}$, $L = N$ to construct a function $g$ expressible by a $3$-layer
	network of width at most $\frac{8c_{\sigma}}{\delta}\alpha^{3/2}N d^{11/4}+1$, satisfying $\sup_{\bx \in \reals^d} |g(\bx) - h(\bx)| \leq \delta$.
	This implies that $\lmnorm{g-h } \leq \delta$, and moreover, that the range of $g$ is in $[-1-\delta,1+\delta]\subseteq [-2,+2]$ (since we assume $\delta<1$). Combining with \eqref{eq:est-lip1} and using the triangle inequality finishes the proof.
\end{proof}

\subsection{Finishing the Proof}

We are finally ready to prove our main theorem.
\begin{proof}[Proof of Theorem \ref{thm:main}]
	The proof is a straightforward combination of propositions \ref{prop:main} and \ref{prop:approx} (whose conditions can be verified to follow immediately from the assumptions used in the theorem). We first choose $\alpha = C$ and $N = \lceil C \alpha^{3/2} d^2 \rceil$ with the constant $C$ taken from the statement of Proposition \ref{prop:main}. By invoking this proposition we obtain signs $\epsilon_i \in \{-1,1\}$ and a universal constant $\delta_1 > 0$ for which any function $f$ expressed by a bounded-size $2$-layer network satisfies
	\begin{equation}\label{eq:final1}
	\lmnorm{ \tilde g - f } \geq \delta_1,
	\end{equation}
	where
	$
	\tilde g(\bx) = \sum_{i=1}^N \epsilon_i g_i(\norm{\bx})$.
	Next, we use Proposition \ref{prop:approx} with $\delta = \min\{\delta_1/2,1\}$ to approximate $\tilde g$ by a function $g$ expressible by a $3$-layer network of width at most
	\[
	\frac{16 c_{\sigma}}{\delta}\alpha^{3/2}N d^{11/4}+1
	~=~ \frac{16 c_{\sigma}}{\delta}C^{3/2}\lceil C^{5/2}d^2\rceil d^{11/4}+1
	~\leq~ C'c_{\sigma}d^{19/4},
	\]
	(where $C'$ is a universal constant depending on the universal constants $C,\delta_1$),
	so that 
	\begin{equation}\label{eq:final2}
	\lmnorm{ \tilde g - g} \leq \delta \leq \delta_1 / 2.
	\end{equation}
	Combining \eqref{eq:final1} and \eqref{eq:final2} with the triangle inequality, we have that $\lmnorm{ f-g } \geq \delta_1 / 2$ for any $2$-layer function $f$. The proof is complete.
\end{proof}

\subsubsection*{Acknowledgements} 

OS is supported in part by an FP7 Marie Curie CIG grant, the Intel ICRI-CI Institute, and Israel Science Foundation grant 425/13. We thank James Martens and the anonymous COLT 2016 reviewers for several helpful comments.

\bibliographystyle{plainnat}
\bibliography{mybib}

\begin{thebibliography}{26}
\providecommand{\natexlab}[1]{#1}
\providecommand{\url}[1]{\texttt{#1}}
\expandafter\ifx\csname urlstyle\endcsname\relax
  \providecommand{\doi}[1]{doi: #1}\else
  \providecommand{\doi}{doi: \begingroup \urlstyle{rm}\Url}\fi

\bibitem[Barron(1994)]{barron1994approximation}
Andrew~R. Barron.
\newblock Approximation and estimation bounds for artificial neural networks.
\newblock \emph{Machine Learning}, 14\penalty0 (1):\penalty0 115--133, 1994.

\bibitem[Bianchini and Scarselli(2014)]{bianchinicomplexity}
M.~Bianchini and F.~Scarselli.
\newblock On the complexity of shallow and deep neural network classifiers.
\newblock In \emph{ESANN}, 2014.

\bibitem[Cohen et~al.(2015)Cohen, Sharir, and Shashua]{cohen2015expressive}
N.~Cohen, O.~Sharir, and A.~Shashua.
\newblock On the expressive power of deep learning: A tensor analysis.
\newblock \emph{arXiv preprint arXiv:1509.05009}, 2015.

\bibitem[Cybenko(1989)]{cybenko1989approximation}
G.~Cybenko.
\newblock Approximation by superpositions of a sigmoidal function.
\newblock \emph{Mathematics of control, signals and systems}, 2\penalty0
  (4):\penalty0 303--314, 1989.

\bibitem[Debao(1993)]{debao1993degree}
C.~Debao.
\newblock Degree of approximation by superpositions of a sigmoidal function.
\newblock \emph{Approximation Theory and its Applications}, 9\penalty0
  (3):\penalty0 17--28, 1993.

\bibitem[Delalleau and Bengio(2011)]{delalleau2011shallow}
O.~Delalleau and Y.~Bengio.
\newblock Shallow vs. deep sum-product networks.
\newblock In \emph{NIPS}, pages 666--674, 2011.

\bibitem[{\relax DLMF}()]{NIST:DLMF}
{\relax DLMF}.
\newblock {NIST Digital Library of Mathematical Functions}.
\newblock http://dlmf.nist.gov/, Release 1.0.10 of 2015-08-07, 2015.
\newblock URL \url{http://dlmf.nist.gov/}.

\bibitem[Funahashi(1989)]{funahashi1989approximate}
Ken-Ichi Funahashi.
\newblock On the approximate realization of continuous mappings by neural
  networks.
\newblock \emph{Neural networks}, 2\penalty0 (3):\penalty0 183--192, 1989.

\bibitem[Grafakos and Teschl(2013)]{grafakos2013fourier}
Loukas Grafakos and Gerald Teschl.
\newblock On fourier transforms of radial functions and distributions.
\newblock \emph{Journal of Fourier Analysis and Applications}, 19\penalty0
  (1):\penalty0 167--179, 2013.

\bibitem[Hajnal et~al.(1993)Hajnal, Maass, Pudl{\'a}k, Szegedy, and
  Tur{\'a}n]{hajnal1993threshold}
Andr{\'a}s Hajnal, Wolfgang Maass, Pavel Pudl{\'a}k, M{\'a}ri{\'o} Szegedy, and
  Gy{\"o}rgy Tur{\'a}n.
\newblock Threshold circuits of bounded depth.
\newblock \emph{Journal of Computer and System Sciences}, 46\penalty0
  (2):\penalty0 129--154, 1993.

\bibitem[H{\aa}stad(1986)]{hastad1986almost}
J.~H{\aa}stad.
\newblock Almost optimal lower bounds for small depth circuits.
\newblock In \emph{Proceedings of the eighteenth annual ACM symposium on Theory
  of computing}, pages 6--20. ACM, 1986.

\bibitem[Hornik et~al.(1989)Hornik, Stinchcombe, and
  White]{hornik1989multilayer}
K.~Hornik, M.~Stinchcombe, and H.~White.
\newblock Multilayer feedforward networks are universal approximators.
\newblock \emph{Neural networks}, 2\penalty0 (5):\penalty0 359--366, 1989.

\bibitem[Hunter and Nachtergaele(2001)]{HNBook}
John~K. Hunter and Bruno Nachtergaele.
\newblock \emph{Applied analysis}.
\newblock World Scientific Publishing Co., Inc., River Edge, NJ, 2001.

\bibitem[Krasikov(2014)]{krasikov2014approximations}
Ilia Krasikov.
\newblock Approximations for the bessel and airy functions with an explicit
  error term.
\newblock \emph{LMS Journal of Computation and Mathematics}, 17\penalty0
  (01):\penalty0 209--225, 2014.

\bibitem[Maass et~al.(1994)Maass, Schnitger, and Sontag]{maass1994comparison}
W.~Maass, G.~Schnitger, and E.~Sontag.
\newblock A comparison of the computational power of sigmoid and boolean
  threshold circuits.
\newblock In V.~P. Roychowdhury, K.~Y. Siu, and A.~Orlitsky, editors,
  \emph{Theoretical Advances in Neural Computation and Learning}, pages
  127--151. Springer, 1994.

\bibitem[Martens and Medabalimi(2014)]{martens2014expressive}
J.~Martens and V.~Medabalimi.
\newblock On the expressive efficiency of sum product networks.
\newblock \emph{arXiv preprint arXiv:1411.7717}, 2014.

\bibitem[Martens(2015)]{martens2015}
James Martens.
\newblock Private Communication, 2015.

\bibitem[Martens et~al.(2013)Martens, Chattopadhya, Pitassi, and
  Zemel]{martens2013representational}
James Martens, Arkadev Chattopadhya, Toni Pitassi, and Richard Zemel.
\newblock On the representational efficiency of restricted boltzmann machines.
\newblock In \emph{NIPS}, pages 2877--2885, 2013.

\bibitem[Montufar et~al.(2014)Montufar, Pascanu, Cho, and
  Bengio]{montufar2014number}
G.~F Montufar, R.~Pascanu, K.~Cho, and Y.~Bengio.
\newblock On the number of linear regions of deep neural networks.
\newblock In \emph{Advances in Neural Information Processing Systems}, pages
  2924--2932, 2014.

\bibitem[Parberry(1994)]{parberry1994circuit}
I.~Parberry.
\newblock \emph{Circuit complexity and neural networks}.
\newblock MIT press, 1994.

\bibitem[Pascanu et~al.(2013)Pascanu, Montufar, and Bengio]{pascanu2013number}
R.~Pascanu, G.~Montufar, and Y.~Bengio.
\newblock On the number of inference regions of deep feed forward networks with
  piece-wise linear activations≫.
\newblock \emph{arXiv preprint arXiv}, 1312, 2013.

\bibitem[Rossman et~al.(2015)Rossman, Servedio, and Tan]{rossman2015average}
B.~Rossman, R.~Servedio, and L.-Y. Tan.
\newblock An average-case depth hierarchy theorem for boolean circuits.
\newblock In \emph{{FOCS}}, 2015.

\bibitem[Rumelhart et~al.(1986)Rumelhart, Hinton, and
  Williams]{rumelhart1986parallel}
D.~E. Rumelhart, G.~E. Hinton, and R.~J. Williams.
\newblock Learning internal representations by error propagation.
\newblock In D.~E. Rumelhart and J.L. McClelland, editors, \emph{Parallel
  distributed Processing}, volume~1. {MIT} Press, 1986.

\bibitem[Shpilka and Yehudayoff(2010)]{shpilka2010arithmetic}
A.~Shpilka and A.~Yehudayoff.
\newblock Arithmetic circuits: A survey of recent results and open questions.
\newblock \emph{Foundations and Trends{\textregistered} in Theoretical Computer
  Science}, 5\penalty0 (3--4):\penalty0 207--388, 2010.

\bibitem[Sodin(2007)]{Sodin05}
Sasha Sodin.
\newblock {Tail-Sensitive Gaussian Asymptotics for Marginals of Concentrated
  Measures in High Dimension}.
\newblock In Vitali~D. Milman and Gideon Schechtman, editors, \emph{Geometric
  Aspects of Functional Analysis}, volume 1910 of \emph{Lecture Notes in
  Mathematics}, pages 271--295. Springer Berlin Heidelberg, 2007.

\bibitem[Telgarsky(2015)]{telgarsky2015representation}
M.~Telgarsky.
\newblock Representation benefits of deep feedforward networks.
\newblock \emph{arXiv preprint arXiv:1509.08101}, 2015.

\end{thebibliography}

\newpage

\appendix

\section{Approximation Properties of the ReLU Activation Function}\label{sec:relu}

In this appendix, we prove that the ReLU activation function satisfies assumption \ref{assumption}, and also prove bounds on the Lipschitz parameter of the approximation and the size of the required parameters. Specifically, we have the following lemma:

\begin{lemma} \label{lem:onedimapprox}
	Let $\sigma(z)=\max\{0,z\}$ be the ReLU activation function, and fix $L, \delta,R > 0$. Let $f:\reals\rightarrow\reals$ which is constant outside an interval $[-R,R]$.
	There exist scalars $a,\{\alpha_i, \beta_i\}_{i=1}^{w}$, where $w\leq 3\frac{RL}{\delta}$, such that the function 
	\begin{equation} \label{eq:hform}
	h(x) = a + \sum_{i=1}^w \alpha_i \sigma(x - \beta_i)
	\end{equation}
	is $L$-Lipschitz and satisfies
	\begin{equation} \label{eq:fhclose}
	\sup_{x \in \reals} \bigl |f(x) - h(x) \bigr | \leq \delta.
	\end{equation}
	Moreover, one has $|\alpha_i| \leq 2 L$ and $w \leq 3 \frac{R L}{\delta}$.
\end{lemma}
\begin{proof}
	If one has $2RL < \delta$, then the results holds trivially because we may take the function $h$ to be the $0$ function (with width parameter $w=0$). Otherwise, we must have $R\geq \delta/2L$, so by increasing the value of $R$ by a factor of at most $2$, we may assume without loss of generality that there exists an integer $m$ such that $R = m \delta / L$. 
	
	Let $h$ be the unique piecewise linear function which coalesces with $f$ on points of the form $\delta / L i$, $i \in \mathbb{Z} \cap [-m,m]$, is linear in the intervals $(w \delta/L, (w+1) \delta/L)$ and is constant outside $[-R,R]$. Since $f$ is $L$-Lipschitz, equation \eqref{eq:fhclose} holds true. It thus suffices to express $h$ as a function having the form \eqref{eq:hform}. Let $\beta_i = i \delta / L$,
	choose $a = h(-R)$ and set
	$$
	\alpha_i = h'(\beta_i + \tfrac{\delta}{2L}) - h'(\beta_i - \tfrac{\delta}{2L}), ~~ -m \leq i \leq m.
	$$
	Then clearly equation \eqref{eq:hform} holds true. Moreover, we have $|\alpha_i| \leq 2 L$, which completes the proof.
	
\end{proof}

\section{Technical Proofs}\label{sec:technicalproofs}

\subsection{Proof of \lemref{lem:varphi}}\label{subsec:prooflemvarphi}
By definition of the Fourier transform, 
\[
\varphi(\bx) = \int_{\bw:\norm{\bw}\leq R_d}\exp(-2\pi i \bx^\top \bw)d\bw.
\]
Since $\varphi(\bx)$ is radial (hence rotationally invariant), let us assume without loss of generality that it equals $r\be_1$, where $r=\norm{\bx}$ and $\be_1$ is the first standard basis vector. This means that the integral becomes
\[
\int_{\bw:\norm{\bw}\leq R_d}\exp(-2\pi i r w_1)d\bw ~=~\int_{w_1=-R_d}^{R_d}\exp(-2\pi i r w_1)\left(\int_{w_2\ldots w_d:\sum_{j=2}^{d}w_j^2\leq R_d^2-w_1^2}dw_2\ldots dw_d\right)dw_1.
\]
The expression inside the parenthesis is simply the volume of a ball of radius $\left(R_d^2-w_1^2\right)^{1/2}$ in $\reals^{d-1}$. Letting $V_{d-1}$ be the volume of a unit ball in $\reals^{d-1}$, this equals
\[
\int_{w_1=-R_d}^{R_d}\exp(-2\pi i r w_1)\left(V_{d-1}(R_d^2-w_1^2)^{\frac{d-1}{2}}\right)dw_1.
\]
Performing the variable change $z=\arccos(w_1/R_d)$ (which implies that as $w_1$ goes from $-R_d$ to $R_d$, $z$ goes from $\pi$ to $0$, and also $R_d\cos(z)=w_1$ and $-R_d\sin(z)dz = dw_1$), we can rewrite the integral above as
\begin{align*}
	&	V_{d-1}\int_{z=0}^{\pi}\left(R_d^2-R_d^2\cos^2(z)\right)^{\frac{d-1}{2}}\exp(-2\pi i r R_d \cos(z))R_d\sin(z)dz\\
	& = V_{d-1}R_d^d\int_{z=0}^{\pi}\sin^{d}(z)\exp\left(-2\pi i r R_d\cos(z)\right)dz.
\end{align*}
Since we know that this integral must be real-valued (since we're computing the Fourier transform $\varphi(\bx)$, which is real-valued and even), we can ignore the imaginary components, so the above reduces to
\begin{equation}\label{eq:phi1}
	V_{d-1}R_d^d\int_{z=0}^{\pi}\sin^{d}(z)\cos\left(2\pi r R_d\cos(z)\right)dz.
\end{equation}
By a standard formula for Bessel functions (see Equation 10.9.4. in \cite{NIST:DLMF}), we have
\[
J_{d/2}(x) = \frac{(x/2)^{d/2}}{\pi^{1/2}\Gamma\left(\frac{d+1}{2}\right)}\int_{0}^{\pi}\sin^d(z)\cos(x\cos(z))dz,
\]
which by substituting $x=2\pi r R_d$ and changing sides, implies that
\[
\int_{0}^{\pi}\sin^d(z)\cos(2\pi r R_d\cos(z))dz = \frac{\pi^{1/2}\Gamma\left(\frac{d+1}{2}\right)}{(\pi r R_d)^{d/2}}J_{d/2}(2\pi r R_d).
\]
Plugging this back into \eqref{eq:phi1}, we get the expression 
\[
V_{d-1}R_d^{d/2}\frac{\pi^{1/2}\Gamma\left(\frac{d+1}{2}\right)}{(\pi r )^{d/2}}J_{d/2}(2\pi r R_d).
\]
Plugging in the explicit formula $V_{d-1}=\frac{\pi^{(d-1)/2}}{\Gamma\left(\frac{d+1}{2}\right)}$, this simplifies to
\[
\left(\frac{R_d}{r}\right)^{d/2}J_{d/2}(2\pi R_d r).
\]
Recalling that this equals $\varphi(x)$ where $\norm{\bx}=r$, the result follows.

\subsection{Proof of \lemref{lem:flat}}
By \lemref{lem:varphi},
\[
\varphi(x) = \left(\frac{R_d}{x}\right)^{d/2}J_{d/2}(2\pi R_dx).
\]
Moreover, using the definition of a good interval, and the fact that the maximal value in any interval is at most $2\alpha\sqrt{d}$, we have
\begin{equation}\label{eq:flat1}
|J_{d/2}(2\pi R_d x)|\geq \frac{1}{\sqrt{80\pi R_d x}}\geq \frac{1}{\sqrt{160\pi R_d \alpha\sqrt{d}}}.
\end{equation}
Since $x$ (in any interval) is at least $\alpha\sqrt{d}$, then $J_{d/2}(2\pi R_d x)$ is $2\pi R_d$-Lipschitz in $x$ by \lemref{lem:lipmag}. Since the width of each interval only $\frac{\alpha\sqrt{d}}{N}$, \eqref{eq:flat1} implies that $J_{d/2}(2\pi R_d x)$ (and hence $\varphi(x)$) does not change signs in the interval, provided that $N>2\sqrt{160}\left(\pi\alpha R_d \sqrt{d}\right)^{3/2}$. Recalling that $R_d\leq \frac{1}{2}\sqrt{d}$, this is indeed satisfied by the lemma's conditions.

Turning to the second part of the lemma, assuming $\varphi(x)$ is positive without loss of generality, and using the Lipschitz property of $J_{d/2}(\cdot)$ and \eqref{eq:flat1}, we have	
\begin{align*}
\frac{\sup_{x\in \Delta_i}\varphi(x)}{\inf_{x\in \Delta_i}\varphi(x)} &\leq 
\frac{\sup_{x\in \Delta_i}\left(\frac{R_d}{x}\right)^{d/2}}{\inf_{x\in \Delta_i}\left(\frac{R_d}{x}\right)^{d/2}}\cdot \frac{\sup_{x\in \Delta_i}J_{d/2}(2\pi R_d x)}{\inf_{x\in \Delta_i}J_{d/2}(2\pi R_d x)}\\
&\leq \left(\frac{\sup_{x\in \Delta_i}x}{\inf_{x\in \Delta_i}x}\right)^{d/2}\cdot\frac{\inf_{x\in \Delta_i}J_{d/2}(2\pi R_d x)+\frac{2\pi R_d \alpha \sqrt{d}}{N}}{\inf_{x\in \Delta_i}J_{d/2}(2\pi R_d x)}\\
&\leq \left(\frac{\inf_{x\in \Delta_i}x + \frac{\alpha\sqrt{d}}{N}}{\inf_{x\in \Delta_i}x}\right)^{d/2}\left(1+\frac{2\pi R_d\alpha\sqrt{d}}{N}\sqrt{80\pi R_d \alpha \sqrt{d}}\right)\\
&\leq \left(1+\frac{\alpha\sqrt{d}}{N\alpha\sqrt{d}}\right)^{d/2}\left(1+\frac{2\sqrt{80}(\pi \alpha R_d \sqrt{d})^{3/2}}{N}\right)\\
&\leq
\left(1+\frac{1}{N}\right)^{d/2}\left(1+\frac{2\sqrt{80}(\pi\alpha d/2)^{3/2}}{N}\right),
\end{align*}
which is less than $1+d^{-1/2}$ provided that $N\geq c\alpha^{3/2}d^2$ for some universal constant $c$.

\subsection{Proof of \lemref{lem:nothinsh}}

The result is trivially true for a bad interval $i$ (where $g_i$ is the $0$ function, hence both sides of the inequality in the lemma statement are $0$), so we will focus on the case that $i$ is a good interval.

For simplicity, let us denote the interval $\Delta_i$ as $[\ell,\ell+\delta]$, where
$\delta=\frac{1}{N}$ and $\ell$ is between $\alpha \sqrt{d}$ and $2\alpha \sqrt{d}$. Therefore, the conditions in the lemma imply that  $\delta\leq \frac{1}{50d\ell}$. Also, we drop the $i$ subscript and refer to $g_i$ as $g$.

Since, $g$ is a radial function, its Fourier transform is also radial, and is given by
\[
\hat{g}(\bw)=\hat{g}(\norm{\bw}) =  2\pi\int_{s=0}^{\infty}g(s)\left(\frac{s}{\norm{\bw}}\right)^{d/2-1}J_{d/2-1}(2\pi s\norm{\bw}) s~ds,
\]
(see for instance \cite{grafakos2013fourier}, section 2, and references therein). Using this formula, and switching to polar coordinates (letting $A_{d}$ denote the surface area of a unit sphere in $\reals^d$), we have the following:
\begin{align}
&\int_{2R_d B_d}\hat{g}^2(\bw) d\bw~=~ \int_{r=0}^{2R_d}A_{d}r^{d-1}\hat{g}^2(r) dr\notag\\
&= \int_{r=0}^{2R_d}A_{d}r^{d-1}\left(2\pi\int_{s=0}^{\infty}g(s)\left(\frac{s}{r}\right)^{d/2-1}J_{d/2-1}(2\pi sr) s~ds\right)^2 dr\notag\\
&=4\pi^2 A_{d}\int_{r=0}^{2R_d}r\left(\int_{s=0}^{\infty}g(s)s^{d/2}J_{d/2-1}(2\pi sr) ~ds\right)^2dr\notag\\
&=4\pi^2 A_{d}\int_{r=0}^{2R_d}r\left(\int_{s=\ell}^{\ell+\delta}s^{d/2}J_{d/2-1}(2\pi sr) ~ds\right)^2dr.\label{eq:bess1}
\end{align}
By \lemref{lem:lipmag}, $|J_{d/2-1}(x)|\leq 1$, hence \eqref{eq:bess1} can be upper bounded by
\begin{align*}
&4\pi^2A_{d}\int_{r=0}^{2R_d}r\left(\int_{s=\ell}^{\ell+\delta}s^{d/2}ds\right)^2dr
~\leq~ 4\pi^2A_{d}\int_{r=0}^{2R_d}r\left(\delta(\ell+\delta)^{d/2}\right)^{2}dr\\
&\leq 4\pi^2 A_{d}\delta^2(\ell+\delta)^d\int_{r=0}^{2R_d}r~dr
~=~ 8\pi^2 A_{d}\delta^2(\ell+\delta)^d R_d^2.
\end{align*}
Overall, we showed that
\begin{equation}\label{eq:ghatup}
\int_{2R_dB_d}\hat{g}^2(\bw) d\bw~\leq~ 8\pi^2R_d^2 A_{d}\delta^2(\ell+\delta)^d. 
\end{equation}
Let us now turn to consider $\int\hat{g}^2(\bw)d\bw$, where the integration is over all of $\bw\in\reals^d$. By isometry of the Fourier transform, this equals $\int g^2(\bx)d\bx$, so
\[
\int\hat{g}^2(\bw)d\bw~=~\int_{\reals^d}g^2(\bx)d\bx~=~\int_{r=0}^{\infty}A_{d}r^{d-1}g^2(r)dr ~=~\int_{r=\ell}^{\ell+\delta}A_{d}r^{d-1}dr ~\geq~
A_{d}\delta\ell^{d-1}.
\]
Combining this with \eqref{eq:ghatup}, we get that
\[
\frac{\int_{2R_dB_d}\hat{g}^2(\bw) d\bw}{\int_{\reals^d}\hat{g}^2(\bw)d\bw}
~\leq~ \frac{8\pi^2R_d^2 A_{d}\delta^2(\ell+\delta)^d}{A_{d}\delta\ell^{d-1}}
~=~ 8\pi^2R_d^2\ell\delta\left(1+\frac{\delta}{\ell}\right)^d.
\]
Since we assume $\delta\leq \frac{1}{50d\ell}$, and it holds that $\left(1+\frac{1}{50d}\right)^{d}\leq \exp(1/50)$ and $R_d\leq \frac{1}{2}\sqrt{d}$ by \lemref{lem:Rd}, the above is at most
\[
2\pi^2 d\ell\delta\left(1+\frac{1}{50d}\right)^{d} ~\leq~
2\pi^2 d\ell\delta\exp(1/50) ~\leq~ 2\pi^2\frac{1}{50}\exp(1/50) < \frac{1}{2}.
\]
Overall, we showed that $\frac{\int_{2R_d B_d}\hat{g}^2(\bw) d\bw}{\int_{\reals^d}\hat{g}^2(\bw)d\bw}\leq \frac{1}{2}$, and therefore
\[
\frac{\int_{(2R_d B_d)^{C}}\hat{g}^2(\bw) d\bw}{\int_{\reals^d}\hat{g}^2(\bw)d\bw} ~=~
\frac{\int_{\reals^d}\hat{g}^2(\bw)d\bw-\int_{2R_d B_d}\hat{g}^2(\bw) d\bw}{\int_{\reals^d}\hat{g}^2(\bw)d\bw} ~\geq~ 1-\frac{1}{2} ~=~ \frac{1}{2} 
\]
as required.

\subsection{Proof of \lemref{lem:nothinsh2}}

The result is trivially true for a bad interval $i$ (where $g_i$ is the $0$ function, hence both sides of the inequality in the lemma statement are $0$), so we will focus on the case that $i$ is a good interval.

Define $a = \sup_{x \in \Delta_i} \varphi(x)$. Using Lemma \ref{lem:flat}, we have that $\varphi(x)$ does not change signs in the interval $\Delta_i$. Suppose without loss of generality that it is positive. Moreover, by the same lemma we have that
$$
|\varphi(x) - a| \leq d^{-1/2} a, ~~ \forall x \in \Delta_i
$$ 
Consequently, we have that
\begin{align} \label{eq:phia}
\int_{(2 R_d B_d)^C} (\widehat{((\varphi-a) g_i)} (\bw))^2 d \bw ~& \leq \int_{\reals^d}(\widehat{((\varphi-a) g_i)} (\bw))^2 d \bw \\
& = \int_{\reals^d}((\varphi-a) g_i (x))^2 d x  \nonumber \\
& \leq d^{-1} \int_{\reals^d}(a g_i (x))^2 d x	\nonumber.
\end{align}
Next, by choosing the constant $C$ to be large enough, we may apply Lemma \ref{lem:nothinsh}, which yields that
\begin{equation} \label{eq:lem4}
\int_{(2 R_d B_d)^C} (\widehat{(a g_i)} (\bw))^2 d \bw \geq \frac{1}{2} \int_{\reals^d}(a g_i(x))^2 d x.
\end{equation}
By the triangle inequality, we have that for two vectors $u,v$ in a normed space, one has $\norm{v}^2 \geq \norm{u}^2 - 2 \norm{v} \norm{v-u}$. This teaches us that
\begin{align*}
\int_{(2 R_d B_d)^C} (\widehat{(g_i \varphi)} (\bw))^2 d \bw ~& \geq \int_{(2 R_d B_d)^C} (\widehat{(a g_i)} (\bw))^2 d \bw \\
&- 2 \sqrt{ \int_{(2 R_d B_d)^C} (\widehat{(a g_i)} (\bw))^2 d \bw} \sqrt{ \int_{(2 R_d B_d)^C} (\widehat{((\varphi-a) g_i)} (\bw))^2 d \bw} \\
& \stackrel{\eqref{eq:phia} }{\geq} \int_{(2 R_d B_d)^C} (\widehat{(a g_i)} (\bw))^2 d \bw - 2d^{-1/2} \int_{\reals^d}(a g_i (x))^2 d x \\
& \stackrel{\eqref{eq:lem4}}{\geq} \frac{1}{2} (1 - 4 d^{-1/2}) \int_{\reals^d}(a g_i(x))^2 d x \geq \frac{1}{4} \int_{\reals^d}(\varphi(x) g_i(x))^2 d x. 
\end{align*}

\subsection{Proof of \lemref{lem:bigmass}}

Since the $g_i$ for different $i$ have disjoint supports (up to measure-zero sets), the integral in the lemma equals
\[
\int\sum_{i=1}^{N}\left(\epsilon_i g_i(\bx)\right)^2\varphi^2(\bx)d\bx
~=~\int\sum_{i=1}^{N}g_i^2(\bx)\varphi^2(\bx)d\bx~=~
\int_{\bx:\norm{\bx}\in \text{good}~\Delta_i}\varphi^2(\bx)d\bx,
\]
where we used the definition of $g_i$. Switching to polar coordinates (letting $A_{d}$ be the surface area of the unit sphere in $\reals^d$), and using the definition of $\varphi$ from \lemref{lem:varphi}, this equals
\[
A_{d}\int_{r\in \text{good}~\Delta_i}r^{d-1}\varphi^2(r)dr
~=~ A_{d}\int_{r\in \text{good}~\Delta_i}\frac{R_d^d}{r}J_{d/2}^2(2\pi R_d r)dr\\ 
\]
Recalling that $A_{d} = \frac{d\pi^{d/2}}{\Gamma\left(\frac{d}{2}+1\right)}$ and that $R_d^d = \pi^{-d/2}\Gamma\left(\frac{d}{2}+1\right)$ by \lemref{lem:Rd}, this equals
\begin{equation}\label{eq:indimp1}
d\int_{r\in \text{good}~\Delta_i}\frac{J_{d/2}^2(2\pi R_d r)}{r}dr.
\end{equation}
We now claim that for any  $r\in [\alpha\sqrt{d},2\alpha \sqrt{d}]$ (that is, in any interval),
\begin{equation}\label{eq:indimp2}
J_{d/2}^2(2\pi R_d r)\geq \frac{1}{40\pi R_d r} ~~~\Longrightarrow~~~ r\in \text{~good $\Delta_i$},
\end{equation}
which would imply that we can lower bound \eqref{eq:indimp1} by 
\begin{equation}\label{eq:indimp3}
d\int_{\alpha \sqrt{d}}^{2\alpha \sqrt{d}}\frac{J_{d/2}^2(2\pi R_d r)}{r}\ind{J_{d/2}^2(2\pi R_d r)\geq \frac{1}{40\pi R_d r}}dr.
\end{equation}
To see why \eqref{eq:indimp2} holds, consider an $r$ which satisfies the left hand side. The width of its interval is at most $\frac{\alpha\sqrt{d}}{N}$, and by \lemref{lem:lipmag}, $J_{d/2}(2\pi R_d r)$ is at most $2\pi R_d$-Lipschitz in $r$. Therefore, for any other $r'$ in the same interval as $r$, it holds that
\[
\left|J_{d/2}(2\pi R_d r')\right|\geq \sqrt{\frac{1}{40\pi R_d r}}-\frac{2\pi R_d\alpha\sqrt{d}}{N},
\]
which can be verified to be at least $\sqrt{\frac{1}{80\pi R_d r}}$ by the condition on $N$ in the lemma statement, and the facts that $r\leq 2\alpha \sqrt{d}, R_d\leq \frac{1}{2}\sqrt{d}$. As a result, $J_{d/2}^2(2\pi R_d r')\geq \frac{1}{80\pi R_d r}$ for any $r'$ in the same interval as $r$, which implies that $r$ is in a good interval. 

We now continue by taking \eqref{eq:indimp3}, and performing the variable change $x=2\pi R_d r$, leading to
\[
d\int_{2\pi R_d\alpha\sqrt{d}}^{4\pi R_d\alpha\sqrt{d}}\frac{J_{d/2}^2(x)}{x}\ind{J_{d/2}^2(x)\geq \frac{1}{20x}}dx.
\]
Applying \lemref{lem:besind} with $\beta=2\pi R_d\alpha/\sqrt{d}$ (which by \lemref{lem:Rd}, is between $2\pi\alpha/5$ and $\pi\alpha$, hence satisfies the  conditions of \lemref{lem:besind} if $\alpha$ is large enough), this is at least
\[
d\frac{0.005}{\beta d}~\geq~\frac{0.005}{2\pi\alpha/5}~\geq~ \frac{0.003}{\alpha},
\]
from which the lemma follows.

\subsection{Proof of \lemref{lem:lipapprox}}
For any $i$, define
\[
\check{g}_i(x) = \begin{cases}\max\{1,N\text{dist}(x,\Delta_i^C)\} & \text{$i$ good}\\ 0 & \text{$i$ bad}\end{cases}
\]
where $\text{dist}(x,\Delta_i^C)$ is the distance of $x$ from the boundaries of $\Delta_i$. Note that for bad $i$, this is the same as $g_i(x)$, whereas for good $i$, it is an $N$-Lipschitz approximation of $g_i(x)$.

Let $f(\bx)=\sum_{i=1}^{N}\epsilon \check{g}(\bx)$, and note that since the support of $\check{g}_i$ are disjoint, $f$ is also $N$ Lipschitz. With this definition, the integral in the lemma becomes
\[
\int\left(\sum_{i=1}^{N}\epsilon_i(\check{g}_i(\bx)-g_i(\bx))\right)^2\varphi^2(\bx)d\bx.
\]
Since the support of $\check{g}_i(\bx)-g_i(\bx)$ is disjoint for different $i$, this equals
\[
\int\sum_{i=1}^{N}\left(\check{g}_i(\bx)-g_i(\bx)\right)^2\varphi^2(\bx)d\bx~=~
\sum_{i=1}^{N}\int\left(\check{g}_i(\bx)-g_i(\bx)\right)^2\varphi^2(\bx)d\bx.
\]
Switching to polar coordinates (using $A_{d}$ to denote the surface area of the unit sphere in $\reals^d$), and using the definition of $\varphi$ from \lemref{lem:varphi}, this equals
\[
\sum_{i=1}^{N}\int_{0}^{\infty}A_d r^{d-1}(\check{g}_i(r)-g_i(r))^2\varphi^2(r)dr
~=~ \sum_{i=1}^{N}\int_{0}^{\infty} A_d\frac{R_d^d}{r}(\check{g}_i(r)-g_i(r))^2J_{d/2}^2(2\pi R_d r)dr.
\]
Using the definition of $R_d$ from \lemref{lem:Rd}, and the fact that $A_{d}=\frac{d\pi^{d/2}}{\Gamma\left(\frac{d}{2}+1\right)}$, this equals
\[
\sum_{i=1}^{N}\int_{0}^{\infty} \frac{d}{r}(\check{g}_i(r)-g_i(r))^2J_{d/2}^2(2\pi R_d r)dr.
\]
Now, note that by definition of $\check{g_i},g_i$, their difference $|\check{g}_i(r)-g_i(r)|$ can be non-zero (and at most $1$) only for $r$ belonging to two sub-intervals of width $\frac{1}{N}$ within the interval $\Delta_i$ (which itself lies in $[\alpha\sqrt{d},2\alpha\sqrt{d}]$). Moreover, for such $r$ (which is certainly at least $\alpha\sqrt{d}$), we can use \lemref{lem:besbound} to upper bound $J_{d/2}^2(2\pi R_d r)$ by $\frac{1.3}{\alpha d}$. Overall, we can upper bound the sum of integrals above by
\[
\sum_{i=1}^{N}\frac{d}{\alpha\sqrt{d}}\cdot \frac{2}{N}\cdot \frac{1.3}{\alpha d}
~<~ \frac{3}{\alpha^2\sqrt{d}}.
\]

\section{Technical Results On Bessel functions}\label{sec:bessel}

\begin{lemma}\label{lem:lipmag}
	For any $\nu\geq 0$ and $x$, $|J_{\nu}(x)|\leq 1$. Moreover, for any $\nu\geq 1$ and $x\geq 3\nu$, $J_{\nu}(x)$ is $1$-Lipschitz in $x$.
\end{lemma}
\begin{proof}
	The bound on the magnitude follows from equation 10.14.1 in \cite{NIST:DLMF}. 
	
	The derivative of $J_{\nu}(x)$ w.r.t. $x$ is given by $-J_{\nu+1}(x)+(\nu/x)J_{\nu}(x)$ (see equation 10.6.1 in \cite{NIST:DLMF}). Since $|J_{\nu+1}(x)|$ and $|J_{\nu}(x)|$, for $\nu\geq 1$, are at most $\frac{1}{\sqrt{2}}$ (see equation 10.14.1 in \cite{NIST:DLMF}), we have that the magnitude of the derivative is at most $\frac{1}{\sqrt{2}}\left|1+\frac{\nu}{x}\right| \leq \frac{1}{\sqrt{2}}\left(1+\frac{1}{3}\right) < 1$.
\end{proof}

To prove the lemmas below, we will need the following explicit approximation result for the Bessel function $J_{d/2}(x)$, which is an immediate corollary of 
Theorem 5 in \cite{krasikov2014approximations}, plus some straightforward approximations (using the facts that for any $z\in (0,0.5]$, we have $\sqrt{1-z^2}\geq 1-0.3z$ and $0\leq z\arcsin(z)\leq 0.6z$):

\begin{lemma}[\cite{krasikov2014approximations}]\label{lem:besapprox}
	If $d\geq 2$ and $x\geq d$, then
	\[
	\left|J_{d/2}(x)-\sqrt{\frac{2}{\pi c_{d,x} x}}\cos\left(-\frac{(d+1)\pi}{4}+f_{d,x}x\right)\right| ~\leq~ x^{-3/2},
	\]
	where
	\[
	c_{d,x} = \sqrt{1-\frac{d^2-1}{4x^2}}~~~,~~~ f_{d,x}=c_{d,x}+\frac{\sqrt{d^2-1}}{2x}\arcsin\left(\frac{\sqrt{d^2-1}}{2x}\right).
	\]
	Moreover, assuming $x\geq d$,
	\[
	1\geq c_{d,x} \geq 1-\frac{0.15~d}{x} \geq 0.85
	\]
	and
	\[
	1.3\geq 1+\frac{0.3~d}{x}\geq f_{d,x} \geq 1-\frac{0.15~d}{x} \geq 0.85 
	\]
\end{lemma}

\begin{lemma}\label{lem:besbound}
	If $d\geq 2$ and $r\geq \sqrt{d}$, then
	\[
	J_{d/2}^2(2\pi R_d r) \leq \frac{1.3}{r\sqrt{d}}.
	\]
\end{lemma}
\begin{proof}
	Using \lemref{lem:besapprox} (which is justified since $r\geq \sqrt{d}$ and $R_d\geq \frac{1}{5}\sqrt{d}$ by \lemref{lem:Rd}), the fact that $\cos$ is at most $1$, and the assumption $d\geq 2$, 
	\begin{align*}
		\left|J_{d/2}(2\pi R_d r)\right| &\leq \sqrt{\frac{2}{\pi\cdot 0.85\cdot 2\pi R_d r}}
		+(2\pi R_d r)^{-3/2}\\
		&= \frac{1}{\sqrt{2\pi R_d r}}\left(\sqrt{\frac{2}{0.85\pi}}+\frac{1}{2\pi R_d r}\right)\\
		&\leq
		\sqrt{\frac{5}{2\pi \sqrt{d}r}}\left(\sqrt{\frac{2}{0.85\pi}}+\frac{5}{2\pi \sqrt{d}\sqrt{d}}\right)\\
		&\leq \sqrt{\frac{5}{2\pi \sqrt{d}r}}\left(\sqrt{\frac{2}{0.85\pi}}+\frac{5}{4\pi}\right)
	\end{align*}
	Overall, we have that
	\[
	J_{d/2}^2(2\pi R_d r) \leq \frac{5}{2\pi r\sqrt{d}}\left(\sqrt{\frac{2}{0.85\pi}}+\frac{5}{4\pi}\right)^2 ~\leq~ \frac{1.3}{r\sqrt{d}}.
	\]
\end{proof}

\begin{lemma}\label{lem:besind}
	For any $\beta\geq 1,d\geq 2$ such that $\beta d\geq 127$, it holds that
	\[
	\int_{\beta d}^{2\beta d}~\frac{J_{d/2}^2(x)}{x} \cdot\ind{J_{d/2}^2(x)\geq \frac{1}{20x}}~dx ~\geq~ \frac{0.005}{ \beta d}.
	\]
\end{lemma}
\begin{proof}
	For any $a,b\geq 0$, we have $a\cdot\ind{a\geq b} ~\geq~ a-b$. Therefore,
	\begin{align*}
		&\int_{\beta d}^{2\beta  d}~\frac{1}{x}\cdot J_{d/2}^2(x)\cdot\ind{J_{d/2}^2(x)\geq \frac{1}{20x}}~dx\\
		&\geq
		\int_{\beta d}^{2\beta  d}~\frac{1}{x}\cdot\left( J_{d/2}^2(x)-\frac{1}{20x}\right)dx\\
		&=\int_{\beta d}^{2\beta  d}~\frac{1}{x}J_{d/2}^2(x)dx-\frac{1}{20}\int_{\beta  d}^{2\beta  d}\frac{1}{x^2}dx\\
		&=
		\int_{\beta d}^{2\beta  d}~\frac{1}{x}J_{d/2}^2(x)dx-\frac{1}{40\beta  d}.
	\end{align*}
	We now wish to use \lemref{lem:besapprox} and plug in the approximation for $J_{d/2}(x)$. To do so, let $a=J_{d/2}(x)$, let $b$ be its approximation from \lemref{lem:besapprox}, and let $\epsilon=x^{-3/2}$ the bound on the approximation from the lemma.
	Therefore, we have $|a-b|\leq \epsilon$. This implies 
	\begin{equation}\label{eq:ab}
		a^2\geq b^2-(2|b|+\epsilon)\epsilon,
	\end{equation}
	which follows from
	\[
	b^2-a^2 = (b+a)(b-a)\leq (|b|+|a|)|b-a| \leq (|b|+|b|+\epsilon)\epsilon = (2|b|+\epsilon)\epsilon.
	\]
	\eqref{eq:ab} can be further simplified, since by definition of $b$ and \lemref{lem:besapprox},
	\[
	|b|\leq \sqrt{\frac{2}{\pi c_{d,x}x}}\leq \sqrt{\frac{2}{\pi\cdot 0.85\cdot x}} \leq \frac{1}{\sqrt{x}}.
	\]
	Plugging this back into \eqref{eq:ab}, plugging in the definition of $a,b$, and recalling that $c_{d,x}\leq 1$ and $x\geq d\geq 2$, we get that
	\begin{align*}
		J_{d/2}^2(x) &\geq \frac{2}{\pi c_{d,x} x}\cos^2\left(-\frac{(d+1)\pi}{4}+f_{d,x}x\right)-\left(\frac{2}{\sqrt{x}}+x^{-3/2}\right)x^{-3/2}\\
		&\geq \frac{2}{\pi  x}\cos^2\left(-\frac{(d+1)\pi}{4}+f_{d,x}x\right)-3x^{-2}.
	\end{align*}
	Therefore,
	\begin{align*}
		&\int_{\beta d}^{2\beta  d}~\frac{1}{x}J_{d/2}^2(x)dx \\
		&\geq
		\frac{2}{\pi}\int_{\beta d}^{2\beta  d}~\frac{1}{ x^2}\cos^2\left(-\frac{(d+1)\pi}{4}+f_{d,x}x\right)dx-3\int_{\beta d}^{2\beta  d}x^{-3}dx\\
		&=
		\frac{2}{\pi}\int_{\beta d}^{2\beta  d}~\frac{1}{ x^2}\cos^2\left(-\frac{(d+1)\pi}{4}+f_{d,x}x\right)dx-\frac{9}{8\beta ^2d^2}.
	\end{align*}
	To compute the integral above, we will perform a variable change, but first lower bound the integral in a more convenient form. A straightforward calculation (manually or using a symbolic computation toolbox) reveals that
	\begin{align*}
		\frac{\partial}{\partial x}\left(f_{d,x}x\right)
		&=	\sqrt{1-\frac{d^2-1}{4x^2}},
	\end{align*}
	which according to \lemref{lem:besapprox}, equals $c_{d,x}$, which is at most $1$. Using this and the fact that $f_{d,x}\geq 0.85$ by the same lemma ,
	\begin{align*}
		\int_{\beta d}^{2\beta  d}&~\frac{1}{ x^2}\cos^2\left(-\frac{(d+1)\pi}{4}+f_{d,x}x\right)dx\\
		&\geq 
		\int_{\beta d}^{2\beta  d}~\frac{1}{ x^2}\cos^2\left(-\frac{(d+1)\pi}{4}+f_{d,x}x\right)\left(
		\frac{\partial}{\partial x}\left(f_{d,x}x\right)\right)dx\\
		&\geq 
		\int_{\beta d}^{2\beta  d}~\frac{0.85^2}{ (f_{d,x}x)^2}\cos^2\left(-\frac{(d+1)\pi}{4}+f_{d,x}x\right)\left(
		\frac{\partial}{\partial x}\left(f_{d,x}x\right)\right)dx\\		
	\end{align*}
	Using the variable change $z=f_{d,x} x$, and the fact that $1.3\geq f_{d,x}\geq 0.85$, the above equals
	\[
	0.85^2 \int_{f_{d,\beta  d}\beta d}^{f_{d,2\beta  d}2\beta  d}~\frac{1}{ z^2}\cos^2\left(-\frac{(d+1)\pi}{4}+z\right)dz
	~\geq~ 0.85^2 \int_{1.3 \beta  d}^{1.7 \beta  d}~\frac{1}{ z^2}\cos^2\left(-\frac{(d+1)\pi}{4}+z\right)dz
	\]
	We now perform integration by parts. Note that $\cos^2\left(-\frac{(d+1)\pi}{4}+z\right) = \frac{\partial}{\partial z}\left(\frac{z}{2}+\frac{1}{4}\sin\left(-\frac{(d+1)\pi}{2}+2z\right)\right)$, and $\sin$ is always bounded by $1$, hence
	\begin{align*}
		\int_{1.3\beta  d}^{1.7\beta  d}&~\frac{1}{ z^2}\cos^2\left(-\frac{(d+1)\pi}{4}+z\right)dz\\
		&= \frac{\frac{z}{2}+\frac{1}{4}\sin\left(-\frac{(d+1)\pi}{2}+2z\right)}{z^2}~\Big|_{1.3\beta  d}^{1.7\beta  d}
		+2\int_{1.3\beta  d}^{1.7\beta  d}\frac{\frac{z}{2}+\frac{1}{4}\sin\left(-\frac{(d+1)\pi}{2}+2z\right)}{z^3}dz\\
		&\geq \left(\frac{1}{2z}+\frac{\sin\left(-\frac{(d+1)\pi}{2}+2z\right)}{4z^2}\right)~\Big|_{1.3\beta  d}^{1.7\beta  d}+\int_{1.3\beta  d}^{1.7\beta  d}\left(\frac{1}{z^2}-\frac{1}{2z^3}\right)dz\\
		&= \left(\frac{1}{2z}+\frac{\sin\left(-\frac{(d+1)\pi}{2}+2z\right)}{4z^2}\right)~\Big|_{1.3\beta  d}^{1.7\beta  d}+\left(-\frac{1}{z}+\frac{1}{4z^2}\right)~\Big|_{1.3\beta  d}^{1.7\beta  d}\\
		&=
		\left(-\frac{1}{2z}+\frac{1+\sin\left(-\frac{(d+1)\pi}{2}+2z\right)}{4z^2}\right)~\Big|_{1.3\beta  d}^{1.7\beta  d}\\
		&=\left(-\frac{1}{2z}\right)~\Big|_{1.3\beta  d}^{1.7\beta  d}+
		\left(\frac{1+\sin\left(-\frac{(d+1)\pi}{2}+2z\right)}{4z^2}\right)~\Big|_{1.3\beta  d}^{1.7\beta  d}\\
		&\geq \left(\frac{0.09}{\beta  d}\right)~~~-~~~\frac{1+1}{4(1.3\beta  d)^2}\\
		&=\frac{1}{\beta  d}\left(0.09-\frac{1}{3.38 \beta  d}\right).
	\end{align*}
	Concatenating all the lower bounds we attained so far, we showed that
	\begin{align*}
		\int_{\beta d}^{2\beta d}&~\frac{1}{x}\cdot J_{d/2}^2(x)\cdot\ind{J_{d/2}^2(x)\geq \frac{1}{9x}}~dx\\
		&\geq
		-\frac{1}{40 \beta  d}-\frac{9}{8\beta ^2 d^2}+\frac{2}{\pi}0.85^2\frac{1}{\beta  d}\left(0.09-\frac{1}{3.38\beta  d}\right)\\
		&\geq \frac{1}{\beta  d}\left(0.015-\frac{1.27}{\beta  d}\right).
	\end{align*}
	If $\beta d \geq 127$, this is at least $\frac{0.005}{\beta  d}$, from which the lemma follows.
\end{proof}

\end{document}